\documentclass[11pt]{article}
\usepackage[pdftex]{graphicx}
\usepackage[T1]{fontenc}
\usepackage{lmodern}

\usepackage{fullpage}

\usepackage{amsmath,amsthm,amssymb}
\usepackage{verbatim}
\usepackage{mathtools}
\usepackage{algorithm}
\usepackage{algorithmicx}
\usepackage{algpseudocode}
\usepackage{subfig}
\usepackage[toc,page]{appendix}
\usepackage[usenames,dvipsnames]{color}

\usepackage{caption}

\usepackage{thmtools}
\usepackage{thm-restate}

\usepackage{hyperref}

\usepackage{cleveref}

\makeatletter
\newcommand{\multiline}[1]{%
  \begin{tabularx}{\dimexpr\linewidth-\ALG@thistlm}[t]{@{}X@{}}
    #1
  \end{tabularx}
}
\makeatother

\global\long\def\norm#1{\big\|#1\big\|}

\usepackage[utf8]{inputenc} 
\usepackage[T1]{fontenc}    
\usepackage{hyperref}       
\usepackage{url}            
\usepackage{booktabs}       
\usepackage{amsfonts}       
\usepackage{nicefrac}       
\usepackage{microtype}      

\usepackage{xcolor}         
\usepackage[utf8]{inputenc} 
\usepackage[T1]{fontenc}    
\usepackage{hyperref}       
\usepackage{url}            
\usepackage{booktabs}       
\usepackage{amsfonts}       
\usepackage{nicefrac}       
\usepackage{microtype}      
\usepackage{lipsum}
\usepackage{amsmath,amsthm,amssymb}
\usepackage{verbatim}
\usepackage{mathtools}
\usepackage{algorithm}
\usepackage{algorithmicx}
\usepackage{algpseudocode}
\usepackage{subfig}
\usepackage[toc,page]{appendix}

\usepackage{caption}

\usepackage{thmtools}
\usepackage{thm-restate}

\usepackage{cleveref}

\usepackage{tikz}

\crefname{algocf}{alg.}{algs.}
\Crefname{algocf}{Algorithm}{Algorithms}

\declaretheorem[name=Theorem,numberwithin=section]{thm}
\declaretheorem[name=Theorem,numberlike=thm]{theorem}
\declaretheorem[name=Lemma,numberlike=thm]{lem}
\declaretheorem[name=Lemma,numberlike=thm]{lemma}

\declaretheorem[name=Corollary,numberlike=thm]{cor}

\declaretheorem[name=Definition,numberlike=thm,style=definition]{defn}

\global\long\def\norm#1{\big\|#1\big\|}

\newcommand{\defeq}{\stackrel{\mathrm{def}}{=}}
\newcommand{\R}{\mathbb{R}}

\newcommand{\Z}{\mathbb{Z}}
\newcommand{\argmax}{\mathrm{argmax}}

\newcommand{\mx}{\textbf{X}}
\newcommand{\expo}[1]{\exp\left(#1\right)}
\newcommand{\prob}[1]{\mathbb{P}\left(#1\right)}

\newcommand{\bp}{\textbf{p}}

\newcommand{\mtz}{{\textbf{Z}}}
\newcommand{\mw}{{\textbf{W}}}
\newcommand{\summ}[5]{#1(#2:#3,#4:#5)}
\newcommand{\summw}[4]{\summ{\mw}{#1}{#2}{#3}{#4}}
\newcommand{\summx}[4]{\summ{\mx}{#1}{#2}{#3}{#4}}
\newcommand{\summy}[4]{\summ{\my}{#1}{#2}{#3}{#4}}
\newcommand{\transpara}[8]{#1,#2,#3,#4,#5,#6,#7,#8}
\newcommand{\transname}{\mathrm{trans}}
\newcommand{\trans}[3]{\transname(#1,#2,#3)}
\newcommand{\twodinterval}[4]{[#1:#2,#3:#4]}
\newcommand{\proundx}{\mathrm{partial\_round\_special}}
\newcommand{\msplitname}{\mathrm{split}}
\newcommand{\mcombinename}{\mathrm{combine}}
\newcommand{\msplit}[2]{\msplitname(#1,#2)}
\newcommand{\mcombine}[2]{\mcombinename(#1,#2)}

\newcommand{\mc}{\textbf{C}}
\newcommand{\my}{\textbf{Y}}
\newcommand{\vones}{\textbf{1}}

\newcommand{\mj}{\textbf{m}_{j}}
\newcommand{\mja}{\textbf{m}_{j_1}}
\newcommand{\mjb}{\textbf{m}_{j_2}}

\newcommand{\boo}{b_1}
\newcommand{\btt}{b_2}

\newcommand{\ztbtt}{[0,\btt]}

\newcommand{\bttpo}{(\btt+1)}

\newcommand{\F}{\phi}

\newcommand{\1}{\overrightarrow{1}}
\newcommand{\bbP}{\mathbb{P}}

\newcommand{\simplex}{\Delta^{\bX}}
\newcommand{\psimplex}{\Delta_{\mathrm{pseudo}}^{\bX}}
\newcommand{\dsimplex}{\Delta_{\bR}^{\bX}}

\newcommand{\probpml}{\bbP}

\newcommand{\bg}{\textbf{g}}

\newcommand{\bq}{\textbf{q}}
\newcommand{\bff}{\textbf{f}}

\newcommand{\bX}{\mathcal{D}}

\newcommand{\bZ}{\textbf{Z}}
\newcommand{\bR}{\textbf{R}}

\newcommand{\aaron}[1]{{\small \textsc{\color{green} Aaron: #1}}}

\newcommand{\sidford}[1]{\aaron{#1}}

\definecolor{colorx}{rgb}{0.4, 0.1, 0.7}

\newcommand{\otilde}{\widetilde{O}}

\newcommand{\cphi}{C_{\phi}}

\newcommand{\ma}{\textbf{A}}
\newcommand{\mb}{\textbf{B}}

\newcommand{\pvec}{\zeta}

\newcommand{\onevec}{\overrightarrow{\mathrm{1}}}

\newcommand{\setd}{\textbf{M}}
\newcommand{\eled}{\textbf{m}}

\newcommand{\mz}{\textbf{m}_{0}}

\newcommand{\bZfrac}{\textbf{Z}^{\phi,\mathrm{frac}}_{\bR}}

\newcommand{\ri}{\textbf{r}_{i}}
\newcommand{\ria}{\textbf{r}_{i_1}}
\newcommand{\rib}{\textbf{r}_{i_2}}

\newcommand{\bigO}[1]{O\left(#1 \right)}

\newcommand{\rones}{\textbf{1}}
\newcommand{\cones}{\textbf{1}}
\newcommand{\vr}{\textbf{r}}

\newcommand{\rmin}{\textbf{r}_{\mathrm{min}}}

\newcommand{\figurepack}[3]
{\begin{figure}
    \centering
    #3
    \caption{#2}
    #1
\end{figure}}

\title{On the Efficient Implementation of High Accuracy Optimality of Profile Maximum Likelihood}

\author{
  Moses Charikar\\
  Stanford University\\
  \texttt{moses@cs.stanford.edu}
   \\
   \and
   Zhihao Jiang\\
  Stanford University\\
  \texttt{faebdc@stanford.edu} \\
     \and
Kirankumar Shiragur\\
  Stanford University\\
  \texttt{shiragur@stanford.edu} \\
     \and
Aaron Sidford\\
  Stanford University\\
  \texttt{sidford@stanford.edu}
}

\begin{document}

\maketitle

\begin{abstract}%
We provide an efficient unified plug-in approach for estimating symmetric properties of distributions given $n$ independent samples. Our estimator is based on profile-maximum-likelihood (PML) and is sample optimal for estimating various symmetric properties when the estimation error $\epsilon \gg n^{-1/3}$. This result improves upon the previous best accuracy threshold of $\epsilon \gg n^{-1/4}$ achievable by polynomial time computable PML-based universal estimators \cite{ACSS20, ACSS20b}. Our estimator reaches a theoretical limit for universal symmetric property estimation as \cite{Han20} shows that a broad class of universal estimators (containing many well known approaches including ours) cannot be sample optimal for every $1$-Lipschitz property when $\epsilon \ll n^{-1/3}$.
\end{abstract}

\newpage

\section{Introduction}
Given $n$ independent samples $y_1,...,y_n \in \bX$ from an unknown discrete distribution $\bp \in \simplex$ the problem of estimating properties of $\bp$, e.g.\ entropy, distance to uniformity, support size and coverage are among the most fundamental in statistics and learning. Further, the problem of estimating \emph{symmetric properties} of distributions $\bp$ (i.e.\ properties invariant to label permutations) are well studied and have numerous applications~ \cite{Chao84, BF93, CCGLMCL12,TE87, Fur05, KLR99, PBGELLSD01, DS13, RCSWTKRWC09, GTPB07, HHRB01}. 

Over the past decade, symmetric property estimation has been studied extensively and there have been 
many improvements to the time and sample complexity for estimating different properties, e.g.\ support~\cite{VV11a, WY15}, coverage~\cite{ZVVKCSLSDM16,OSW16}, entropy~\cite{VV11a, WY16, JVHW15}, and distance to uniformity~\cite{VV11b, JHW16}. Towards unifying the attainment of computationally-efficient, sample-optimal estimators a striking work of \cite{ADOS16} provided a universal plug-in approach based on a (approximate) profile maximum likelihood (PML) distribution, that (approximately) maximizes the likelihood of the observed profile (i.e. multiset of observed frequencies). 

Formally, \cite{ADOS16} showed that given  $y_1,...,y_n$ if there exists an estimator for a symmetric property $f$ achieving accuracy $\epsilon$ and failure probability $\delta$, then this PML-based plug-in approach achieves error $2\epsilon$ with failure probability $\delta \exp{(3\sqrt{n})}$. As the failure probability $\delta$ for many estimators for well-known properties (e.g.\ support size and coverage, entropy, and distance to uniformity) is roughly $\exp{(-\epsilon^2n)}$, this result implied a sample optimal unified approach for estimating these properties when the estimation error $\epsilon \gg n^{-1/4}$. 

This result of \cite{ADOS16} laid the groundwork for a line of work on the study of computational and statistical aspects of PML-based approaches to symmetric property estimation. For example, follow up work of \cite{HS20} improved the analysis of \cite{ADOS16} and showed that the failure probability of PML is at most $\delta^{1-c}\exp(-n^{1/3+c})$, for any constant $c>0$ and therefore it is sample optimal in the regime $\epsilon \gg n^{-1/3}$. The condition $\epsilon \gg n^{-1/3}$ on the optimality of PML is tight~\cite{Han20}, in the sense that, PML is known to be not sample optimal in the regime $\epsilon \ll n^{-1/3}$. In fact, no estimator (that obeys some mild conditions), is sample optimal for estimating all symmetric properties in the regime $\epsilon \ll n^{-1/3}$; see \Cref{sec:results} after \Cref{lem:appprop} for details.

We also remark that the statistical guarantees in \cite{ADOS16, HS20} hold for any $\beta$-approximate PML\footnote{ $\beta$-approximate PML is a distribution that achieves a multiplicative $\beta$-approximation to the PML objective.} for suitable values of $\beta$. In particular, \cite{HS20} showed that any $\beta$-approximate PML for  $\beta > \exp(-n^{1-c'})$ and any constant $c'>0$, has a failure probability of $\delta^{1-c}\exp(n^{1/3+c}+n^{1-c'})$ for any constant $c>0$. These results further imply a sample optimal estimator in the regime $\epsilon \gg n^{- \min(1/3,c'/2)}$ for properties with failure probability less than $\exp(-\epsilon^2n)$. Note that better approximation leads to a larger range of $\epsilon$ for which the estimator is sample optimal.

Regarding computational aspects of PML,
\cite{CSS19} provided the first efficient algorithm with a non-trivial approximation guarantee of 
$\exp(-n^{2/3} \log n)$, which further implied a sample optimal universal estimator for $\epsilon \gg n^{-1/6}$. This was then improved by \cite{ACSS20} which showed how to efficiently compute PML to higher accuracy of $\exp(-\sqrt{n} \log n)$ thereby achieving a sample optimal universal estimator in the regime $\epsilon \gg n^{-1/4}$. The current best polynomial time approximate PML algorithm by \cite{ACSS20b} achieves an accuracy of $\exp(-k \log n)$, where $k$ is the number of distinct observed frequencies. Although this result achieves better instance based statistical guarantees, in the worst case it still only implies a sample optimal universal estimator in the regime $\epsilon \gg n^{-1/4}$.

In light of these results, a key open problem is to close the gap between the regimes $\epsilon \gg n^{-1/3}$ and $\epsilon \gg n^{-1/4}$, where the former is the regime in which PML based estimators are statistically optimal and the later is the regime where efficient PML based estimators exist. In this work we ask:
\begin{center}
\emph{Is there an efficient approximate PML-based estimator that is sample optimal for $\epsilon\gg n^{-1/3}$}.
\end{center}
In this paper, we answer this question in the affirmative. In particular, we give an efficient PML-based estimator that has failure probability at most $\delta^{1-c}\exp(n^{1/3+c}+n^{1-c'})$, and consequently is sample optimal in the regime $\epsilon \gg n^{-1/3}$. As remarked, this result is tight in the sense that PML and a broad class of estimators are known to be not optimal in the regime $\epsilon \ll n^{-1/3}$.

To obtain this result we depart slightly from the previous approaches in \cite{ADOS16,CSS19,ACSS20}. Rather than directly compute an approximate PML distribution we compute a weaker notion of approximation which we show suffices to get us the desired universal estimator. We propose a notion of a $\beta$-\emph{weak approximate PML distribution} inspired by \cite{HS20} and show that an $\exp(-n^{1/3} \log n)$-weak approximate PML achieves the desired failure probability of $\delta^{1-c}\exp(n^{1/3+c})$ for any constant $c>0$.
Further, we provide an efficient algorithm to compute an $\exp(-n^{1/3} \log n)$-weak approximate PML distribution. Our paper can be viewed as an efficient algorithmic instantiation of \cite{HS20}.

Ultimately, our algorithms use the convex relaxation presented in \cite{CSS19,ACSS20} and provide a new rounding algorithm. 
We differ from the previous best $\exp(-k \log n)$ approximate PML algorithm~\cite{ACSS20b} only in the matrix rounding procedure which controls the approximation guarantee. At a high level, the approximation 
guarantee for the rounding procedure in \cite{ACSS20b} is exponential in the sum of matrix dimensions.
In the present work, we need to round a rectangular matrix with an approximation exponential in the smaller dimension, which may be infeasible for arbitrary matrices. Our key technical innovation is to introduce a {\em swap} operation (see \Cref{sec:swapround}) which facilitates such an approximation guarantee. In addition to a better approximation guarantee than \cite{ACSS20b}, our algorithm also exhibits better run times
(see \Cref{sec:comp}).

\paragraph{Organization:} We introduce preliminaries in \Cref{sec:prelim}. In \Cref{sec:results}, we state our main results and also cover related work. In \Cref{sec:convex}, we provide the convex relaxation to PML studied in \cite{CSS19,ACSS20}. Finally, in \Cref{sec:algorithm}, we provide a proof sketch of our main computational result. Many proofs are then differed to the appendix.

\subsection{Preliminaries}
\label{sec:prelim}

\paragraph{General notation:} For matrices $\ma,\mb \in \R^{s \times t}$, we use $\ma \leq \mb$ to denote that $\ma_{ij} \leq \mb_{ij}$ for all $i \in [s]$ and $j \in [t]$. We let  $[a,b]$ and $[a,b]_{\R}$ denote the interval $\geq a$ and $\leq b$ of integers and reals respectively. We use $\otilde(\cdot)$, $\widetilde{\Omega}(\cdot)$ notation to hide all polylogarithmic factors in $n$ and $N$. We let $a_n \gg b_n$ to denote that $a_n \in \Omega(b_n n^{c})$ or $b_n \in O(n^{-c}a_n)$, for some small constant $c>0$. 

Throughout this paper, we assume we receive a sequence of $n$ independent samples from a distribution $\bp \in \simplex$, where $\simplex \defeq \{\bq \in [0,1]_{\R}^{\bX} | \norm{q}_1 = 1\}$ is the set of all discrete distributions supported on domain $\bX$. 
Let $\bX ^n$ be the set of all length $n$ sequences of elements of $\bX$ and for $y^n \in \bX^n$ let $y^n_{i}$ denoting its $i$th element. Let $\bff(y^n,x)\defeq |\{i\in [n] ~ | ~ y^n_i = x\}|$ and $\bp_{x}$ be the frequency and probability of $x\in \bX$ respectively. For a sequence $y^n \in \bX^n$, let $\setd=\{ \bff(y^n,x) \}_{x \in \bX} \backslash \{0\}$ be the set of all its non-zero distinct frequencies and $\eled_1,\eled_2,\dots, \eled_{|\setd|}$ be these distinct frequencies. The \emph{profile} of a sequence $y^n$ denoted $\phi=\Phi(y^n)$ is a vector in $\Z^{|\setd|}$, where $\F_j \defeq |\{x\in \bX ~|~\bff(y^n,x)=\eled_{j} \}|$ is the number of domain elements with frequency $\eled_{j}$. We call $n$ the length of profile $\F$ and let $\Phi^n$ denote the set of all profiles of length $n$. The probability of observing sequence $y^n$ and profile $\phi$ for distribution $\bp$ are
$\bbP(\bp,y^n) = \prod_{x \in \bX}\bp_x^{\bff(y^n,x)}$ and $\probpml(\bp,\phi)=\sum_{\{y^n \in \bX^n~|~ \Phi (y^n)=\phi \}} \bbP(\bp,y^n)$.

\paragraph{Profile maximum likelihood:}
A distribution $\bp_{\phi} \in \simplex$ is a \emph{profile maximum likelihood} (PML) distribution for profile $\phi \in \Phi^{n}$ if $\bp_{\phi} \in \argmax_{\bp \in \simplex} \probpml(\bp,\phi)$. 
	Further, a distribution $\bp^{\beta}_{\phi}$ is a $\beta$-\emph{approximate PML} distribution if $\probpml(\bp^{\beta}_{\phi},\phi)\geq \beta \cdot \probpml(\bp_{\phi},\phi)$.	
For a distribution $\bp$ and a length $n$, let $\mx$ be a random variable that takes value $\phi \in \Phi^n$ with probability $\prob{\bp,\phi}$. We call $H(\mx)$ (entropy of $\mx$) the \emph{profile entropy} with respect to $(\bp,n)$ and denote it by $H(\Phi^n,\bp)$.

\paragraph{Probability discretization:} Let $\bR \defeq \{\ri\}_{i \in [1,\ell]}$ be a finite discretization of the probability space, where $\ri \in [0,1]_{\R}$ and $\ell\defeq |\bR|$. We call $\bq \in [0,1]^{\bX}_{\R}$ a \emph{pseudo-distribution} if $\|\bq\|_1 \leq 1$ and a \emph{discrete pseudo-distribution} with respect to $\bR$ if all its entries are in $\bR$ as well. We use $\psimplex$ and $\dsimplex$ to denote the set of all pseudo-distributions and discrete pseudo-distributions with respect to $\bR$ respectively. In our work, we use the following most commonly used~\cite{CSS19,ACSS20,ACSS20b} probability discretization set. For any $\alpha>0$,
\begin{equation}\label{eq:probdisc}
    \bR_{n,\alpha} \defeq \left\{ 1 \right\} \cup \left\{\frac{1}{2n^2}(1+n^{-\alpha})^i~|~ \text{for all }i \in \Z_{\geq 0} \text{ such that }  \frac{1}{2n^2}(1+n^{-\alpha})^i\leq 1 \right\}~.
\end{equation}
For all probability terms defined involving distributions $\bp$, we extend those definitions to pseudo distributions $\bq$ by replacing $\bp_{x}$ with $\bq_{x}$ everywhere. 

\paragraph{Optimal sample complexity}
The sample complexity of an estimator $\widehat{f}: \bX^n \rightarrow \R$ when estimating a property $f : \simplex \rightarrow \R$ for distributions in a collection $P \subseteq \simplex$, is the number of samples $\widehat{f}$ needs to determine $f$ with accuracy $\epsilon$ and low failure probability $\delta$ for all distributions in $P$. Specifically, 
$$C^{\widehat{f}}(f, P,\delta,\epsilon) \defeq \min \{ n ~|~ \prob{|f(p) - \widehat{f}(x^n)| \geq \epsilon} \leq \delta \text{ for all } p \in P\}.$$ 
The sample complexity of estimating $f$ is the lowest sample complexity of any estimator, 
$$C^{*}(f,P,\delta,\epsilon) =\min_{\widehat{f}} C^{\widehat{f}}(f, P,\delta,\epsilon).$$
In the paper, the dependency on $\delta$ is typically de-emphasized and $\delta$ is assumed to be $1/3$.
\renewcommand{\bp}{\textbf{p}}

\section{Results}\label{sec:results}
Here we provide our main results. In our first result (\Cref{thm:statmain}), we show that a weaker notion of approximate PML suffices to obtain the desired universal estimator. Later we show that these weaker approximate PML distributions can be efficiently computed (\Cref{thm:algmain}).

\begin{defn}
Given a profile $\phi$, we call a distribution $\bp' \in \simplex$ \emph{$\beta$-approximate PML distribution with respect to $\bR$} if 
 $\probpml\left(\bp',\phi\right) \geq \beta \cdot \max_{\bq \in \dsimplex} \probpml\left(\frac{\bq}{\|\bq\|_1},\phi\right)~.$
\end{defn}

The above definition generalizes $\beta$-approximate PML distributions which is simply the special case when $\bR = [0,1]_{\R}$. Using our new definition, we show that for a specific choice of the discretization set $\bR_{n,1/3}$, a distribution $\bp'$ that is an approximate PML with respect to $\bR_{n,1/3}$ suffices to obtain a universal estimator; this result is formally stated below.
\begin{theorem}[Competitiveness of an approximate PML w.r.t $\bR$]\label{thm:statmain}
For symmetric property $f$, suppose there exists an estimator $\widehat{f}$ that takes input a profile $\phi \in \Phi^n$ drawn from $\bp \in \simplex$ and satisfies,
$$\probpml \left(|f(\bp) -\widehat{f}(\phi)| \geq \epsilon  \right) \leq \delta~,$$
then for $\bR =\bR_{n,1/3}$ (See \Cref{eq:probdisc}), a discrete pseudo distribution $\bq' \in \dsimplex$ such that $\bq'/\|\bq'\|_1$ is an $\exp(-O(|\bR|\log n))$-approximate PML distribution with respect to the $\bR$ satisfies,
\begin{equation}\label{eq:statthm}
\probpml \left( \left|f\left(\frac{\bq'}{\|\bq'\|_1}\right) -f(\bp)\right| \geq 2 \epsilon  \right) \leq \delta^{1-c} \exp(O(n^{1/3+c})),~~~~ \text{ for any constant }c>0~.    
\end{equation}
\end{theorem}
The proof of the above theorem is implicit in the analysis of~\cite{HS20}, however we provide a short simpler proof using their continuity lemma (\Cref{lem:cont}). Note that the bound on the failure probability we get is the same asymptotically as that of exact PML from \cite{HS20}, which is known to be tight~\cite{Han20}. Furthermore, to achieve such an improved failure probability bound all we need is an approximate PML distribution with respect to $\bR$, for some $\bR$ which is of small size. Taking advantage of this fact and building upon \cite{CSS19, ACSS20}, we provide a new rounding algorithm that outputs the desired approximate PML distribution with respect to $\bR$.
\begin{theorem}[Computation of an approximate PML w.r.t $\bR$]\label{thm:algmain}
We provide an algorithm that given a probability discretization set $\bR=\bR_{n,\alpha}$ for $\alpha>0$ (See \Cref{eq:probdisc}) and a profile $\phi$ with $k$ distinct frequencies, runs in time $\widetilde{O}\left(|\bR|+\frac{n}{\min(k,|\bR|)}\left(\min(|\bR|,n/k)k^{\omega}+ \min(|\bR|,k)k^{2}\right)\right)$, where $\omega < 2.373$ is the current matrix multiplication constant \cite{will12, Gall14a, AlmanW21} and returns a pseudo distribution $\bq' \in \dsimplex$ such that,
$$\probpml\left(\frac{\bq'}{\|\bq'\|_1},\phi\right) \geq \exp\left(-O(\min(k,|\bR|) \log n) \right) \cdot \max_{\bq \in \dsimplex} \probpml\left(\frac{\bq}{\|\bq\|_1},\phi\right)~.$$
\end{theorem}
When $\bR=\bR_{n,1}$, our algorithm computes an $\exp(-O(k \log n))$ approximate PML distribution, therefore our result is at least as good as the previous best known approximate PML algorithm due to \cite{ACSS20b}. In comparison to \cite{ACSS20b}, our rounding algorithm is simpler and we suspect, more practical. We provide a more detailed comparison to it later in this section.

\paragraph{Applications:}
Our main results have several applications which we discuss here. First note that, combining \Cref{thm:statmain} and \ref{thm:algmain} immediately yields the following corollary. 
\begin{cor}[Efficient unified estimator]\label{cor:universal_estimator}
Given a profile $\phi \in \Phi^n$ with $k$ distinct frequencies, we can compute an approximate PML distribution $\bq'$ that satisfies \Cref{eq:statthm} in \Cref{thm:statmain} in time $\widetilde{O}\left(\frac{n}{\min(k,n^{1/3})}\left(\min(n^{1/3},n/k)k^{\omega}+ \min(n^{1/3},k)k^{2}\right)\right)$.
\end{cor}
For many symmetric properties the failure probability is exponentially small as stated below.
\begin{lemma}[Lemma 2 in \cite{ADOS16}, Theorem 3 in \cite{HS20}]
For distance to uniformity, entropy, support size and coverage, and sorted $\ell_1$ distance there exists an estimator that is sample optimal and the failure probability is at most $\exp(-\epsilon^2 n^{1-\alpha})$ for any constant $\alpha>0$.
\end{lemma}
The above result combined with \Cref{cor:universal_estimator}, immediately yields the following theorem.
\begin{theorem}[Efficient sample optimal unified estimator]\label{lem:appprop}
There exists an efficient approximate PML-based estimator that for $\epsilon \gg n^{-1/3}$ and symmetric properties such as, distance to uniformity, entropy, support size and coverage, and sorted $\ell_1$ distance achieves optimal sample complexity and has failure probability upper bounded by $\exp(-n^{1/3})$.
\end{theorem}
As our work computes an $\exp(-O(k\log n))$ approximate PML, we recover efficient version of Lemma 2.3 and Theorem 2.4 from \cite{ACSS20b}. The first result uses $\exp(-O(k\log n))$ approximate PML algorithm to efficiently implement an estimator that has better statistical guarantees based on profile entropy~\cite{HO20pentropy} (See \Cref{sec:prelim}). The second result provides an efficient implementation of the PseudoPML estimators~\cite{CSS19pseudo,HO19}. Please refer to the respective papers for further details.
\paragraph{Tightness of our result:}\label{subsec:tight}
Recall that \cite{HS20} showed that the failure probability of an (approximate) PML based estimator is upper bounded by $\delta^{1-c} \exp(-n^{1/3+c})$, for any constant $c>0$. This result further implied a sample optimal universal estimator in the regime $\epsilon \gg n^{-1/3}$ for various symmetric properties (\Cref{lem:appprop}). In our work, we efficiently recover these results and a natural question to ask here is if these results can be improved.

As remarked earlier, \cite{Han20} showed that the condition for optimality of PML ($\epsilon \gg n^{-1/3}$) is in some sense tight. More formally, they showed that PML is not sample optimal in estimating every $1$-Lipschitz property in the regime $\epsilon \ll n^{-1/3}$. In fact, the results in \cite{Han20} hold more broadly for any universal plug-in based estimator that outputs a distribution $\hat{\bp}$ satisfying,
$$\max_{\bp \in \simplex} \mathbb{E}\|\bp-\hat{\bp}\|_{1}^{\mathrm{sorted}} \leq A(n)\sqrt{k/n}~,$$
where $A(n) \leq n^{\gamma}$ for every $\gamma>0$ and $\|\bp-\bq\|^{\mathrm{sorted}}_{1} \defeq \min_{\text{permutations } \sigma}\|\bp-\bq_{\sigma}\|_{1}$ denotes the sorted $\ell_1$ distance between $\bp$ and $\bq$. In other words, if an estimator is based on a reasonably good estimate of the true distribution $\bp$ (in terms of sorted-$\ell_1$ distance), then it cannot be sample optimal for every $1$-Lipschitz property. Furthermore, many well-known universal estimators including PML and LLM~\cite{HJW18} indeed provide a reasonably good estimate of the true distribution and therefore cannot be sample optimal in the regime $\epsilon \ll n^{-1/3}$. Please refer to \cite{Han20} for further details.

\paragraph{Comparison to approximate PML algorithms:}\label{sec:comp}
All prior provable approximate PML algorithms~\cite{CSS19,ACSS20,ACSS20b} have two key steps: (Step 1) solve a convex approximation to the PML and (Step 2) round the (fractional) solution to a valid approximate PML distribution.

A convex approximation to PML was first provided in \cite{CSS19} and a better analysis for it is shown in \cite{ACSS20}. In particular, \cite{CSS19} and \cite{ACSS20} showed that an integral optimal solution to step 1 approximates the PML up to accuracy  $\exp(-n^{2/3} \log n)$ and $\exp(-\min(k,|\bR|) \log n)$ respectively, where $k$ and $|\bR|$ are the number of distinct frequencies and distinct probability values respectively. In addition to the loss from convex approximation, the previous algorithms also incurred a loss in the rounding step (Step 2). The loss in the rounding step for the previous works is bounded by $\exp(-n^{2/3}\log n)$~\cite{CSS19}, $\exp(-\sqrt{n}\log n)$~\cite{ACSS20} and $\exp(-k \log n)$~\cite{ACSS20b}.

In our work, we show that there exists a choice of $\bR$ (=$\bR_{n,1/3}$) that is of small size ($|\bR| \leq n^{1/3}$) and suffices to get the desired universal estimator. As $|\bR| \leq n^{1/3}$, our approach only incur a loss of $\exp(-\min(k,|\bR|) \log n) \in \exp(-n^{1/3} \log n)$ in the convex approximation step (Step 1). Furthermore for the rounding step (Step 2), we provide a new simpler and a practical rounding algorithm with a better approximation loss of $\exp(-O(\min(k,|\bR|)\log n) )\in \exp(-O(n^{-1/3}\log n) )$. 

Regarding the run times, both \cite{ACSS20b} and ours have run times of the form $\mathcal{T}_{\mathrm{solve}}+\mathcal{T}_{\mathrm{sparsify}}+\mathcal{T}_{\mathrm{round}}$, where the terms correspond to the time required to solve the convex program, sparsify and round a solution. In our algorithm, we pay the same cost as \cite{ACSS20b} for the first two steps but our run time guarantees are superior to theirs in the rounding step. In particular, the run time of \cite{ACSS20b} is shown as a large polynomial and perhaps not practical as their approach requires enumerating all the approximate min cuts. In contrast, our algorithm has a run time that is subquadratic.

\paragraph{Other related work}
PML was introduced by \cite{OSSVZ04}. Many heuristic approaches have been proposed to compute approximate PML, such as the EM algorithm in \cite{OSSVZ04}, an algebraic approaches in \cite{ADMOP10}, Bethe approximation in \cite{Von12} and \cite{Von14}, and a dynamic programming approach in \cite{PJW17}. For the broad applicability of PML in property testing and to estimate other symmetric properties please refer to \cite{HO19}. Please refer to \cite{HO20pentropy} for details related to profile entropy.
Other approaches for designing universal estimators are: \cite{VV11a} based on \cite{ET76}, \cite{HJW18} based on local moment matching, and variants of PML by \cite{CSS19pseudo,HO19} that weakly depend on the target property that we wish to estimate. Optimal sample complexities for estimating many symmetric properties were also obtained by constructing property specific estimators, e.g. support~\cite{VV11a, WY15}, support coverage~\cite{OSW16,ZVVKCSLSDM16}, entropy~\cite{VV11a, WY16, JVHW15},  distance to uniformity~\cite{VV11b, JHW16}, 
sorted $\ell_{1}$ distance \cite{VV11b, HJW18}, Renyi entropy~\cite{AOST14, AOST17}, KL divergence~\cite{BZLV16, HJW16} and others.

\paragraph{Limitations of our work}
One of the limitations of all the provable approximate PML algorithms~\cite{CSS19,ACSS20,ACSS20b} (including ours) is that they require the solution of a convex program that approximates the PML objective and all these previous works use the CVX solver which is not practical for large sample instances; 
note that 
our results hold for small error regimes which lead
to such large sample instances. Therefore, designing a practical algorithm to solve the convex program is an important future research direction. As discussed above, local moment matching (LLM) based approach is another universal approach for property estimation. It is unclear which of the two (PML or LLM) can lead to practical algorithms.\label{sec:related}

\section{Convex relaxation of PML}\label{sec:convex}
\newcommand{\concat}{\mathrm{concat}}
Here we restate the convex program from \cite{CSS19} that approximates the PML objective. The current best analysis of this convex program is in \cite{ACSS20}. We first describe the notation and later state several results from \cite{CSS19, ACSS20} that capture the guarantees of the convex program.

\paragraph{Notation:} 
For any matrices $\mx\in \R^{a \times c}$ and $\my\in \R^{b \times c}$, we let $\concat(\mx,\my)$ denote the matrix $\mw \in \R^{(a+b) \times c}$ that satisfies, $\mw_{i,j}=\mx_{i,j}$ for all $i \in [a]$ and $j \in [c]$ and $\mw_{a+i,j}=\my_{ij}$ for all $i \in [b]$ and $j \in [c]$. Recall we let $\bR \defeq \{\ri\}_{i \in [\ell]}$ be a finite discretization of the probability space, where $\ri \in [0,1]_{\R}$ and $\ell\defeq |\bR|$. Let $\vr \in [0,1]_{\R}^{\ell}$ be a vector whose $i$'th element is equal to $\ri$. 
The following lemma captures the loss due to probability discretization.
\begin{lem}[Lemma 4.4 in \cite{CSS19}]\label{lem:probdisc}
	Let $\bR =\bR_{n,\alpha}$ for some $\alpha>0$. For any profile $\phi \in \Phi^{n}$ and distribution $\bp \in \simplex$, there exists a pseudo distribution $\bq\in \dsimplex$ that satisfies $\probpml(\bp,\phi) \geq \probpml(\bq,\phi) \geq \expo{-\alpha n-6}\probpml(\bp,\phi)$ and therefore, 
 $$\max_{\bp\in \simplex}\probpml(\bp,\phi) \geq \max_{\bq\in \dsimplex} \probpml(\bq,\phi) \geq \expo{-\alpha n-6}\max_{\bp\in \simplex}\probpml(\bp,\phi)~.$$
\end{lem}
For any probability discretization set $\bR$, profile $\phi$ and pseudo distribution $\bq \in \dsimplex$, define:
\renewcommand{\boo}{\ell}
\renewcommand{\btt}{k}
\renewcommand{\bttpo}{\btt}
\renewcommand{\ztbtt}{[\btt]}
\renewcommand{\pvec}{\textbf{r}}
\renewcommand{\bZ}{\textbf{Z}^{\phi}_{\bR}}

\begin{align}
&\bZ \defeq \Big\{\mx \in \R_{\geq 0}^{\ell \times [0,k]} ~\Big|~ \mx \rones \in \Z^{\ell},  [\mx^\top \cones]_{j}=\phi_{j} \text { for all }j \in [1,k] \text{ and } \vr^{\top} \mx \rones\leq 1\Big\}~,\label{eq:zrphi} \\ 
&\bZfrac \defeq \Big\{\mx \in \R_{\geq 0}^{\ell \times [0,k]}~\Big|~[\mx^\top \cones]_{j}=\phi_{j} \text { for all }j \in [1,k] \text{ and } \vr^{\top} \mx \rones\leq 1 \Big\}~.\label{def:bzfrac}
\end{align}
The $j$'th column corresponds to frequency $m_j$ and we use $\mz \defeq 0$ to capture the unseen elements. Without loss of generality, we assume $\eled_0<\eled_1<\dots< \eled_{k}$. Let $\mc_{ij}\defeq \mj \log \ri$ for all $i \in [\ell]$ and $j \in [0,k]$.  The objective of the optimization problem is follows: for any $\mx \in \R_{\geq 0}^{\ell \times [0,k]}$ define,
\begin{equation}
\bg(\mx)\defeq \exp\Big(\sum_{i\in[\ell],j\in[0,k]}\left[\mc_{ij}\mx_{ij}-\mx_{ij}\log\mx_{ij}\right]+\sum_{i\in[\ell]}[\mx\vones]_{i}\log[\mx\vones]_{i}\Big)~.
\end{equation}
For any $\bq \in \dsimplex$, the function $\bg(\mx)$ approximates the $\probpml(\bq,\phi)$ term and is stated below.
\begin{lem}[Theorem 6.7 and Lemma 6.9 in \cite{ACSS20}]\label{pmlprob:approx}
	Let $\bR$ be a probability discretization set. For any profile $\phi \in \Phi^n$ with $k$ distinct frequencies the following statements hold for $\alpha = \min(k,|\bR|) \log n$:
	$\expo{-O(\alpha)}\cdot	\cphi \cdot	\max_{\mx \in \bZ}\bg(\mx) \leq \max_{\bq \in \dsimplex}\probpml(\bq,\phi) \leq \expo{\bigO{\alpha}} \cdot \cphi \cdot \max_{\mx \in \bZ}\bg(\mx)$ and $\max_{\bq \in \dsimplex}\probpml(\bq,\phi) \leq \expo{\bigO{\min(k,|\bR|) \log n}} \cdot \cphi \cdot \max_{\mx \in \bZfrac}\bg(\mx)~,$
	where $\cphi\defeq \frac{n!}{\prod_{j\in [1,k]}(\eled_{j}!)^{\phi_{j}}}$ is a term that only depends on the profile.\footnote{ The theorem statement in \cite{ACSS20} is only written with an approximation factor of $\exp(O(k\log n) )$. However, their proof provides a stronger approximation factor which is upper bounded by the non-negative rank of the probability matrix, which in turn is upper bounded by the minimum of distinct frequencies and distinct probabilities. Therefore the theorem statement in \cite{ACSS20} holds with a much stronger approximation guarantee of $\expo{\bigO{\min(k,|\bR|) \log n}}$.}
\end{lem}
The proof of concavity for the function $\bg(\mx)$ and a running time analysis to solve the convex program are provided in \cite{CSS19}. For any $\mx \in \bZ$, a pseudo-distributions associated with it is defined below.
\begin{defn}\label{defn:distX}
	For any $\mx \in \bZ$, the discrete pseudo-distribution $\bq_{\mx}$ associated with $\mx$ and $\bR$ is defined as follows: for arbitrary $[\mx \cones]_{i}$ number of domain elements assign probability $\ri$. Further $\bp_{\mx}\defeq \bq_{\mx}/\|\bq_{\mx}\|_1$ is the distribution associated with $\mx$ and $\bR$.
\end{defn}
Note that $\bq_{\mx}$ is a valid pseudo-distribution because of the third condition in \Cref{eq:zrphi} and these pseudo distributions $\bp_{\mx}$ and $\bq_{\mx}$ satisfy the following lemma. 
\begin{lem}[Theorem 6.7 in \cite{ACSS20}]\label{lem:associateddist}
	Let $\bR$ and $\phi \in \Phi^n$ be a probability discretization set and a profile with $k$ distinct frequencies. For any $\mx \in \bZ$, the discrete pseudo distribution $\bq_{\mx}$ and distribution $\bp_{\mx}$ associated with $\mx$ and $\bR$ satisfy: $\expo{-O(k \log n)}\cphi \cdot \bg(\mx) \leq \probpml(\bq_{\mx},\phi) \leq \probpml(\bp_{\mx},\phi)~.$
 \end{lem}

\section{Approximate PML algorithm}\label{sec:algorithm}
\newcommand{\mq}{\textbf{Q}}
\newcommand{\swap}{\mathrm{swap}}
\newcommand{\swapround}{\mathrm{swapmatrixround}}
\newcommand{\mround}{\mathrm{matrixround}}
\newcommand{\mdd}{\textbf{D}}
\newcommand{\res}{\mathrm{residue}}
\newcommand{\create}{\mathrm{create}}
\newcommand{\roundi}{\mathrm{roundiRow}}
\newcommand{\pround}{\mathrm{partialRound}}
\newcommand{\mxfinal}{\mx_{\mathrm{final}}}
\newcommand{\mxnew}{\mx_{\mathrm{new}}}
\newcommand{\mximo}{\mx^{(i-1)}}
\newcommand{\myimo}{\my^{(i-1)}}
\newcommand{\mxi}{\mx^{(i)}}
\newcommand{\myi}{\my^{(i)}}
\newcommand{\mylast}{\my^{(\ell-1)}}
\newcommand{\bRnew}{\textbf{\bR}_{\mathrm{new}}}
\newcommand{\rnew}{\textbf{r}_{\mathrm{new}}}
\newcommand{\bRfinal}{\textbf{\bR}_{\mathrm{final}}}
\newcommand{\bRi}{\textbf{R}^{(i)}}
\newcommand{\bRimo}{\textbf{R}^{(i-1)}}
\newcommand{\bZi}{\textbf{Z}^{\phi,frac}_{\bRi}}
\newcommand{\bZimo}{\textbf{Z}^{\phi,frac}_{\bRimo}}
\newcommand{\bZfinal}{\textbf{Z}^{\phi}_{\bRfinal}}
\newcommand{\mxtemp}{\mx_{\mathrm{temp}}}

\renewcommand{\vones}{\textbf{1}}
\newcommand{\vi}{\textbf{v}^{(i)}}
\newcommand{\vo}{\textbf{v}^{(1)}}
\newcommand{\vl}{\textbf{v}^{(\ell-1)}}
\newcommand{\ww}{\textbf{W}}
\newcommand{\wwi}{\textbf{W}^{(i)}}
\newcommand{\wwimo}{\textbf{W}^{(i-1)}}
\newcommand{\sparsify}{\mathrm{sparsify}}
\newcommand{\discretize}{\mathrm{discretize}}

Here we provide a proof sketch of \Cref{thm:algmain} and provide a rounding algorithm that proves it. Our rounding algorithm takes as input a matrix $\mx \in \bZfrac$ which may have fractional row sums and round it to integral values. This new rounded matrix $\mxfinal$ corresponds to our approximate PML distribution (See \Cref{defn:distX}). The description of our algorithm is as follows. 

\begin{algorithm}[H]
\caption{ApproximatePML($\phi,\bR=\bR_{n,\alpha}$)}
	\begin{algorithmic}[1]
	\State Let $\mx$ be any solution that satisfies,
 $\log \bg (\mx) \geq \max_{\my \in \bZfrac} \log \bg(\my)-\bigO{\min(k,|\bR|) \log n}$.\label{alg:approxpml:line:1} \;
\State $\mx'=\sparsify(\mx)$.\label{alg:approxpml:line:2}\;
\State $(\ma,\mb)=\swapround(\mx')$.\label{alg:approxpml:line:3}
\State $(\mxfinal,\bRfinal)= \create(\ma,\mb,\bR)$\label{alg:approxpml:line:4}
\State Let $\bp'$ be the distribution with respect to $\mxfinal$ and $\bRfinal$ (See \Cref{defn:distX}).\label{alg:approxpml:line:5}\;
	\State Return $\bq=\discretize(\bp',\phi,\bR)$\;
	\end{algorithmic}
	\label{alg:main}
\end{algorithm}

We now provide a guarantee for each of these lines of \Cref{alg:main}. We later use these guarantees to prove our final theorem (\Cref{thm:algmain}). The guarantees of the approximate maximizer $\mx$ computed in the first step of the algorithm are summarized in the following lemma.
\begin{lem}[\cite{CSS19,ACSS20}]\label{lem:maximizer}
	\Cref{alg:approxpml:line:1} of the algorithm can be implemented in $\otilde(|\bR|k^2+|\bR|^2k)$ time and the approximate maximizer $\mx$ satisfies: 
    $$\cphi \cdot \bg(\mx)\geq \expo{-\bigO{\min(k,|\bR|) \log n}} \max_{\bq \in \dsimplex}\probpml(\bq,\phi)~.$$
\end{lem}

The guarantees of the second step of our algorithm are summarized in the following lemma. Please refer to \cite{ACSS20b} for the description of the procedure $\sparsify$. We use this procedure so that we can assume $|\bR| \leq k+1$ as we can ignore the zero rows of the matrix $\mx$.
\begin{lem}[Lemma 4.3 in \cite{ACSS20b}]\label{lem:sparse}
	For any $\mx \in \bZfrac$, the algorithm $\sparsify(\mx)$ runs in time $\otilde(|\bR|~k^{\omega})$ and outputs $\mx' \in \bZfrac$ such that:  
	$\bg(\mx') \geq \bg(\mx) \text{ and } \big|\{i \in [\ell]~|~[\mx'\1]_{i}  > 0\}\big| \leq k+1~.$
\end{lem}
To explain our next step, we need to define a new operation called the $\swap$.
\begin{defn} \label{def:swap}
Given a matrix $\ma$, indices $i_1 <i_2$, $j_1 < j_2$ and a parameter $\epsilon \geq 0$, the operation $\swap(\ma, i_1,i_2,j_1,j_2,\epsilon)$ outputs a matrix $\ma'$ that satisfies,
\begin{equation}
    \ma'_{ij} = 
    \begin{cases}
    \ma_{i,j} + \epsilon  \text{ for } i=i_1,~j=j_1 ~\quad \quad 
    \ma_{i,j} - \epsilon  \text{ for } i=i_1,~j=j_2~,\\
    \ma_{i,j} - \epsilon  \text{ for } i=i_2,~j=j_1 ~\quad \quad
    \ma_{i,j} + \epsilon  \text{ for } i=i_2,~j=j_2~,\\
    \ma_{ij} \text{ otherwise}.
    \end{cases}
\end{equation}
\end{defn}

\begin{defn}[Swap distance]\label{def:swapdis}
$\ma'$ is $x$-swap distance from $\ma$, if $\ma'$ can be obtained from $\ma$ through a sequence of swap operations and the summation of the value $\epsilon$'s in these operations is at most $x$, i.e. there is a set of parameters $\{(i_1^{(s)},i_2^{(s)},j_1^{(s)},j_2^{(s)},\epsilon^{(s)})\}_{s\in [t]}$, where $\sum_{s\in [t]}\epsilon^{(s)}\leq x$, such that $\ma^{(s)}=\swap(\ma^{(s-1)},i_1^{(s)},i_2^{(s)},j_1^{(s)},j_2^{(s)},\epsilon^{(s)})$ for $s\in [t]$, where $\ma^{(0)}=\ma$ and $\ma^{(t)}=\ma'$.
\end{defn}
The following lemma directly follows from \Cref{def:swap} and \Cref{def:swapdis}.

\begin{lemma}  \label{lem:rcsummaintain}
    For any matrices $\ma,\ma' \in \R^{s \times t}$, if $\ma'$ is $x$-swap distance from $\ma$ for some $x\geq 0$, then $\ma' \onevec= \ma \onevec$ and $\ma'^{\top} \onevec= \ma^{\top} \onevec$.
\end{lemma}

Recall that our objective $\bg(\mx)$ contains two terms: (1) the linear term $\sum_{i \in [\ell],j\in [0,k]}\mc_{ij}\mx_{ij}$ and (2) the entropy term $\sum_{i \in [\ell]} [\mx \onevec]_i \log [\mx \onevec]_i- \sum_{i \in [\ell],j\in [0,k]}\mx_{ij} \log \mx_{ij}$. The $\swap$ operation always increases the first term, and in the following lemma we bound the loss due to the second term.
\begin{lemma}\label{lem:swap}
    If $\ma' \in \R^{\ell \times [0,k]}$ is $x$-swap distance from $\ma \in \bZfrac$, then,
    $\ma' \in \bZfrac \text{ and } ~\bg(\ma') \geq \exp(-O(x \log n)) \bg(\ma)$.
\end{lemma}
One of the main contributions of our work is the following lemma, where we repeatedly apply $\swap$ operation to recover a matrix $\ma$ which exhibits several nice properties as stated below.
\begin{lem}\label{lem:matrixround}
For any matrix $\ma \in \R^{s \times t}$ $(s \leq t)$ that satisfies $\ma^{\top} \onevec \in \Z^{t}_{\geq 0}$. The algorithm $\swapround$ runs in $O(s^2 t)$ time and returns matrices $\ma'$ and $\mb$ such that,
\begin{itemize}
    \item $\ma'$ is $O(s)$-swap distance from $\ma$, $\ma' \onevec= \ma \onevec$ and $\ma'^{\top} \onevec= \ma^{\top} \onevec$.
    \item $0 \leq \mb_{ij} \leq \ma'_{ij}$ for all $i \in [s]$ and $j \in [t]$, $\mb \onevec \in \Z_{\geq 0}^{s}$, $\mb^{\top} \onevec \in \Z_{\geq 0}^{t}$ and $\|\ma' -\mb\|_1 \leq O(s)$.
\end{itemize} 
\end{lem}
The above lemma helps us modify our matrix $\mx$ to a new matrix $\ma$ that we can round using the $\create$ procedure. The guarantees of this procedure are summarized below.
\begin{lemma}[Lemma 6.13 in \cite{ACSS20}]\label{lem:create}
	For any $\ma\in \bZfrac \subseteq \R_{\geq 0}^{\ell \times [0,k]}$ and $\mb \in \R_{\geq 0}^{\ell \times [0,k]}$ such that $\mb \leq \ma$, $\mb\1\in \Z^\ell$, $\mb^\top \1\in \Z^{[0,k]}$ and $\|\ma-\mb\|_1\leq t$. The algorithm $\create(\ma,\mb,\bR)$ runs in time $O(\ell k)$ and returns a solution $\ma'$ and a probability discretization set $\bR'$ such that $|\bR'|\leq |\bR| + \min(k+1,t)$, $\ma' \in \textbf{Z}^{\phi}_{\bR'}$ and $\bg(\ma') \geq \expo{-O\left(t\log n\right)}\bg(\ma)~.$
\end{lemma}
As our final goal is to return a distribution in $\dsimplex$, we also use the following discretization lemma.
\begin{lem}\label{lem:disc}
    The function $\discretize$ takes as input a distribution $\bp \in \simplex$ with $\ell'$ distinct probability values, a profile $\phi$, a discretization set of the form $\bR=\bR_{n,\alpha}$ for some $\alpha >0$ and outputs a pseudo distribution $\bq \in \dsimplex$ such that:
    $\prob{\frac{\bq}{\|\bq\|_1},\phi} \geq \exp(-O(\min(k,|\bR|)+\min(k,\ell') + \alpha^2n) \log n) \prob{\bp,\phi}.$
\end{lem}
In \Cref{subsec:one}, we use the guarantees stated above for each line of \Cref{alg:main} to prove \Cref{thm:algmain}. The description of the function $\discretize$ is specified in the proof of \Cref{lem:disc}. We describe the procedure $\swapround$ and provide a proof sketch of \Cref{lem:matrixround} in \Cref{sec:swapround}.

\subsection{Description of $\swapround$ and comparison to \cite{ACSS20b}}\label{sec:swapround}
Here we describe the procedure $\swapround$ and compare our rounding algorithm to \cite{ACSS20b}. Both of \cite{ACSS20b} and our approximate PML algorithm have four main lines (\ref{alg:approxpml:line:1}-\ref{alg:approxpml:line:4}); we differ from \cite{ACSS20b} in the key \Cref{alg:approxpml:line:3}. 
This line in \cite{ACSS20b} invokes a procedure called $\mround$ that takes as input a matrix $\ma \in \R^{\ell \times [0,k]}$ and outputs a matrix $\mb \in \R^{\ell \times [0,k]}$ such that: $\mb \leq \ma$, $\mb \onevec \in \Z^{\ell}_{\geq 0}$, $\mb^\top \onevec \in \Z^{[0,k]}_{\geq 0}$ and $\|\ma-\mb\|_1 \leq O(\ell+k)$. Such a matrix $\mb$ is crucial as the procedure $\create$ uses $\mb$ to round fractional row sums of matrix $\ma$ to integral values. The error incurred in these two steps is at most $\exp(O(\| \ma-\mb\|_1 \log n)) \in \exp(O((\ell+k)\log n))$. As the procedure $\sparsify$ allows us to assume $\ell \leq k+1$, we get an $\exp(-k \log n)$ approximate PML using \cite{ACSS20b}. However, the setting that we are interested in is when $\ell \ll k$; for instance when $\ell \in O(n^{1/3})$ and $k \in \Theta(\sqrt{n})$. In these settings, we desire an $\exp(-O(\min(\ell,k) \log n)) \in \exp(-O(\ell \log n))$ approximate PML. 
In order to get such an improved approximation using \cite{ACSS20b}, we need a matrix $\mb$ satisfying the earlier mentioned inequalities along with $\|\ma-\mb\|_1 \leq O(\min(k,\ell))$. However, such a matrix $\mb$ may not exist for arbitrary matrices $\ma$ and the best guarantee any algorithm can achieve is $\|\ma-\mb\|_1 \in O(\ell+k)$.

To overcome this, we introduce a new procedure called $\swapround$ that takes as input, a matrix $\ma$ and transforms it to a new matrix $\ma'$ that satisfies: $\bg(\ma') \geq \exp(-O(\min(k,\ell) \log n)) \bg(\ma)$. Furthermore, this transformed matrix $\ma'$ exhibits a matrix $\mb$ that satisfies the guarantees: $\mb \leq \ma'$, $\mb \onevec \in \Z^{\ell}_{\geq 0}$, $\mb^\top \onevec \in \Z^{k}_{\geq 0}$ and $\|\ma'-\mb\|_1 \leq O(\ell)$. These matrices $\ma'$ and $\mb$ are nice in that we can invoke the procedure $\create$, which would output a valid distribution with required guarantees. In the following we provide a description of the algorithm that finds these matrices $\ma'$ and $\mb$.
\begin{algorithm}[H]
	\caption{$\swapround(\ma)$}
	\begin{algorithmic}[1]
	\State Let $\ma^{(0)}=\ma$ and $\mdd^{(0)}=0$.
	\For{$r=1\dots \ell$}
	\State $(\my,j) = \pround(\ma^{(r-1)},r)$
	\State $\ma^{(r)} = \roundi (\my,j,r)$.
	\State $\mdd^{(r)}=\mdd^{(r-1)}+\my-\ma^{(r)}$.
	\EndFor
	\State {Return} $\ma'=\mdd^{(\ell)}+\ma^{(\ell)}$ and $\mb=\ma^{(\ell)}$.
	\end{algorithmic}
	\label{alg:proundx}
\end{algorithm}
Our algorithm includes two main subroutines: $\pround$ and $\roundi$. At each iteration $i$, the procedure $\pround$ considers row $i$ and modifies it by repeatedly applying the $\swap$ operation. This modified row is nice as the procedure $\roundi$ can round this row to have an integral row sum while not affecting the rows in $[i-1]$. By iterating through all rows, we get the required matrices $\ma'$ and $\mb$ that satisfy the required guarantees. In the remainder, we formally state the guarantees achieved by the procedures $\pround$ and $\roundi$.

\begin{restatable}{lem}{lempround}\label{lem:partial-round} 
The algorithm $\pround$ takes as inputs $\mx \in \R^{\ell \times [0,k]}_{\geq 0}$ and $i \in [\ell-1]$ that satisfies the following,
$[\mx \onevec]_{i'} \in \Z_{\geq 0} \text{ for all } i'\in [1,i-1] \text{ and } [\mx^{\top} \onevec]_{j} \in \Z_{\geq 0} \text{ for all }j\in [0,k]$,
and outputs a matrix $\my \in \R_{\geq 0}^{\ell \times [0,k]}$ 
and an index $j'$ such that: 
\begin{itemize}
    \item $\my \text{ is within } 3\text{-swap distance from }\mx$.
    \item $\my_{ij'} \geq o \text{ and } \sum_{i'=1}^{i-1} \my_{i'j'}+\my_{ij'}-o \in \Z_{\geq 0}$, where $o=[\mx \onevec]_{i} - \lfloor [\mx \onevec]_{i} \rfloor$.
\end{itemize}
Furthermore, the running time of the algorithm is $O(\ell k)$.

\end{restatable}

Note that by \Cref{lem:rcsummaintain}, if $\my \text{ is within } 3\text{-swap distance from }\mx$, then $\my \onevec= \mx \onevec$ and $\my^{\top} \onevec= \mx^{\top} \onevec$.

\begin{restatable}{lem}{lemroundi}\label{lem:roundi} 
The algorithm $\roundi$ takes as inputs $\my \in \R_{\geq 0}^{\ell \times [0,k]}$, an column index $j \in [0,k]$ and a row index $i \in [\ell-1]$ such that:
$\my^\top \onevec \in \Z^{[0,k]}_{\geq 0}$,  $\my_{ij} \geq o$ and $\sum_{i'=1}^{i-1} \my_{i'j}+\my_{ij}-o \in \Z_{\geq 0}$,
where $o=[\my \onevec]_{i} - \lfloor [\my \onevec]_{i}\rfloor$. Outputs a matrix $\mx \in \R_{\geq 0}^{\ell \times [0,k]}$ such that, \begin{itemize}
    \item $\mx \leq \my$ and $\|\mx-\my\|_1 \leq 1$.
    \item $[\mx \onevec]_{i'}=[\my \onevec]_{i'}\text{ for all } i' \in [i-1],~ [\mx \onevec]_{i} \in \Z_{\geq 0}$, and $\mx^\top \onevec \in \Z^{[0,k]}_{\geq 0}$.
\end{itemize}

\end{restatable}
We defer the description of all the missing procedures and proofs to the appendix.
\section{Proof of Main Result (\Cref{thm:algmain})}\label{subsec:one}

Here we put together the results from the previous sections to prove, \Cref{thm:algmain}.

\begin{proof}[Proof of \Cref{thm:algmain}]
\Cref{alg:main} achieves the guarantees of \Cref{thm:algmain}. In the remainder of the proof, we combine the guarantees of each step of the algorithm to prove the theorem. Toward this end, we first show the following two inequalities: $\mxfinal \in \bZfinal$ and $\bg(\mxfinal) \geq \exp(-O(\min(k,|\bR|) \log n)) \bg(\mx)$.
By \Cref{lem:maximizer}, the \Cref{alg:approxpml:line:1} of \Cref{alg:main} returns a solution $\mx \in \bZfrac$ that satisfies,
\begin{equation}\label{eq:s1}
\cphi \cdot \bg(\mx)\geq \expo{-\bigO{\min(k,|\bR|) \log n}} \max_{\bq \in \dsimplex}\probpml(\bq,\phi)~.    
\end{equation}
By \Cref{lem:sparse}, the \Cref{alg:approxpml:line:2} of \Cref{alg:main} takes input $\mx$ and outputs $\mx'$ such that 
\begin{equation}\label{eq:s2}
\mx' \in \bZfrac \text{ and } \bg(\mx') \geq \bg(\mx),    
\end{equation}
and $\big|\{i \in [\ell]~|~[\mx'\1]_{i}  > 0\}\big| \leq k+1$. As the matrix $\mx'$ has at most $k+1$ non-zero rows, without loss of generality we can assume $|\bR| \leq k+1$ (by discarding zero rows). 

As matrix $\mx' \in \bZfrac$, we have that $\mx'$ has integral column sums and by invoking \Cref{lem:matrixround} with parameters $s=|\bR|$ and $t=k+1$, we get matrices $\ma$ and $\mb$ that satisfy guarantees of \Cref{lem:matrixround}. As $[\ma \onevec]_i = [\mx' \onevec]_i$ for all $i \in [\ell]$, $[\ma^{\top} \onevec]_j = [\mx'^{\top} \onevec]_j$ for all $j \in [0,k]$ and $\mx' \in \bZfrac$, we immediately get that $\ma \in \bZfrac$. Further note that $\ma$ is within $O(|\bR|) = O(\min(|\bR|,k))$-swap distance from $\mx'$ and by \Cref{lem:swap} we get that $\bg(\ma) \geq \exp(-O(\min(|\bR|,k) \log n))\bg(\mx')$. To summarize, we showed the following inequalities, 
\begin{equation}\label{eq:s3}
    \ma \in \bZfrac \text{ and }\bg(\ma) \geq \exp(-O(\min(|\bR|,k) \log n))\bg(\mx')~.
\end{equation}
Note that, \Cref{lem:matrixround} also outputs a matrix $\mb$ that satisfies: $\mb \leq \ma$, $\mb \onevec \in \Z^\ell$, $\mb^\top \onevec\in \Z^{[0,k]}$ and $\|\ma-\mb\|_1 \leq O(\min(|\bR|,k))$. These matrices $\ma$ and $\mb$ satisfy the conditions of \Cref{lem:create} with parameter value $t=O(\min(|\bR|,k))$. Therefore, the procedure $\create$ takes in input matrices $\ma, \mb$ and returns a solution $(\mxfinal,\bRfinal)$ such that $|\bRfinal| \leq |\bR| + \min(\bR,k) \leq 2|\bR|$ and, 
\begin{equation}\label{eq:final}
    \mxfinal \in \bZfinal \text{ and } \bg(\mxfinal) \geq \exp(-O(\min(|\bR|,k)\log n))\bg(\ma)~.
\end{equation} 
As $\mxfinal \in \bZfinal$, by definition \Cref{defn:distX} and \Cref{pmlprob:approx}, the distribution $\bp'$ satisfies,
\begin{align*}
\prob{\bp',\phi} &\geq \exp(-O( \min(k,|\bRfinal|) \log n)) \cphi \bg(\mxfinal)
\geq \exp(-O( \min(k,|\bR|)) \log n)) \cphi \bg(\ma)\\
& \geq \exp(-O( \min(k,|\bR|) \log n)) \cphi \bg(\mx') \geq \exp(-O( \min(k,|\bR|) \log n)) \cphi \bg(\mx)\\
& \geq \exp(-O( \min(k,|\bR|) \log n)) \max_{\bq \in \dsimplex}\probpml(\bq,\phi)~.
\end{align*}
In the second inequality we used \Cref{eq:final} and $|\bRfinal| \leq 2|\bR|$. In the third, fourth and fifth inequalities, we used \Cref{eq:s3}, \Cref{eq:s2} and \Cref{eq:s1} respectively.

Recall we need a distribution that approximately maximizes $\max_{\bq \in \dsimplex}\probpml(\frac{\bq}{\|\bq\|_1},\phi)$ instead of just $\max_{\bq \in \dsimplex}\probpml({\bq},\phi)$. In the remainder of the proof we provide a procedure to output such a distribution.

For any constant $c>0$, let $c\cdot \bR \defeq \{c \cdot \ri ~|~\ri \in \bR \}$. For any $\bq \in \dsimplex$, as $\|\bq\|_1$ satisfies: $\rmin \leq \|\bq\|_1 \leq 1$, we get that,
\begin{equation}\label{eq:a1}
    \max_{\bq \in \dsimplex}\probpml(\frac{\bq}{\|\bq\|_1},\phi) = \max_{c \in [1,1/\rmin]_{\R}}~~ \max_{\bq \in \Delta^{\bX}_{c\cdot \bR}}\probpml(\bq,\phi)~.
\end{equation}
The above expression holds as the maximizer $\bq^*$ of the left hand side satisfies: $\bq^{*} \in \Delta^{\bX}_{(1/\| \bq*\|_1)\cdot \bR}$. Define $C\defeq \{(1+\beta)^i \}_{i \in [a]}$ for some $\beta \in o(1)$, where $a \in O(\frac{1}{\beta}\log (1/\rmin))$ is such that $\rmin (1+\beta)^a=1$. For any constant $c \in [1,1/\rmin]_{\R}$, note that there exists a constant $c' \in C$ such that $c(1-\beta) \leq c' \leq c $. Furthermore, for any distribution $\bq \in \dsimplex$ with $\|\bq\|_1=1/c$, note that the distribution $\bq'=c' \bq \in \Delta^{\bX}_{c'\cdot \bR}$ and satisfies: 
$\probpml(\frac{\bq}{\|\bq\|_1},\phi) =\probpml({c \cdot \bq},\phi)=\probpml(\frac{c}{c'} \bq',\phi)
=\left(\frac{c}{c'}\right)^n\probpml({\bq'},\phi)~.
$ Therefore we get that,
$\probpml({\bq'},\phi) = \left(\frac{c'}{c}\right)^n \probpml(\frac{\bq}{\|\bq\|_1},\phi) \geq (1-\beta)^n\probpml(\frac{\bq}{\|\bq\|_1},\phi) \geq \exp(-2\beta n)\probpml(\frac{\bq}{\|\bq\|_1},\phi) ~.$ Combining this analysis with \Cref{eq:a1} we get that,
\begin{equation}
    \max_{c \in C}~~ \max_{\bq \in \Delta^{\bX}_{c\cdot \bR}}\probpml(\bq,\phi) \geq \exp(-2\beta n)\max_{\bq \in \dsimplex}\probpml(\frac{\bq}{\|\bq\|_1},\phi).
\end{equation}
For each $c>0$ as $|\bR|=|c\cdot \bR|$, our algorithm (\Cref{alg:main}) returns a distribution $\bp_{c}$ that satisfies,
$$\prob{\bp_c,\phi} \geq \exp(-O( \min(k,|\bR|) \log n)) \max_{\bq \in \Delta^{\bX}_{c\cdot \bR}}\probpml(\bq,\phi)~.$$
Let $\bp^{*}$ be the distribution that achieves the maximum objective value to our convex program among the distributions $\{\bp_c\}_{c \in C}$. Then note that $\bp^*$ satisfies: 
$\prob{\bp^*,\phi} \geq \exp(-O( \min(k,|\bR|) \log n)-2 \beta n) \max_{\bq \in \dsimplex}\probpml(\frac{\bq}{\|\bq\|_1},\phi)~.$
Substituting $\beta=\frac{\min(k,|\bR|)}{n}$ in the previous expression, we get,
$$\prob{\bp^*,\phi} \geq \exp(-O( \min(k,|\bR|) \log n)) \max_{\bq \in \dsimplex}\probpml(\frac{\bq}{\|\bq\|_1},\phi)~.$$
As each of our distributions $\bp_c$ (including $\bp^*$) have the number of distinct probability values upper bounded by $2|\bR|$, by applying \Cref{lem:disc}, we get a pseudo distribution $\bq \in \dsimplex$ with the desired guarantees. The  final run time of our algorithm is $O(|C|\mathcal{T}_{1}) \in O(\frac{n}{\min(k,|\bR|)} \cdot \mathcal{T}_{1})$, where $\mathcal{T}_1$ is the time to implement \Cref{alg:main}. Further note that by \Cref{lem:probdisc}, without loss of generality we can assume $|\bR| \leq n/k$. As all the lines of \Cref{alg:main} are polynomial in $n$, our final running time follows from the run times of each line and we conclude the proof.
\end{proof}
\section*{Acknowledgments}
We would like to thank the reviewers for their valuable feedback. Researchers on this project were supported by an Amazon Research Award, a Dantzig-Lieberman Operations Research Fellowship, a Google Faculty Research Award, a Microsoft Research Faculty Fellowship, NSF CAREER Award CCF-1844855, NSF Grant CCF-1955039, a PayPal research gift, a Simons-Berkeley Research Fellowship, a Simons Investigator Award, a Sloan Research Fellowship and a Stanford Data Science Scholarship.    

\bibliography{pml_new}

\newcommand{\etalchar}[1]{$^{#1}$}
\begin{thebibliography}{ADM{\etalchar{+}}10}

\bibitem[ACSS20]{ACSS20b}
Nima Anari, Moses Charikar, Kirankumar Shiragur, and Aaron Sidford.
\newblock Instance based approximations to profile maximum likelihood.
\newblock In Hugo Larochelle, Marc'Aurelio Ranzato, Raia Hadsell,
  Maria{-}Florina Balcan, and Hsuan{-}Tien Lin, editors, {\em Advances in
  Neural Information Processing Systems 33: Annual Conference on Neural
  Information Processing Systems 2020, NeurIPS 2020, December 6-12, 2020,
  virtual}, 2020.

\bibitem[ACSS21]{ACSS20}
Nima Anari, Moses Charikar, Kirankumar Shiragur, and Aaron Sidford.
\newblock The bethe and sinkhorn permanents of low rank matrices and
  implications for profile maximum likelihood.
\newblock In Mikhail Belkin and Samory Kpotufe, editors, {\em Conference on
  Learning Theory, {COLT} 2021, 15-19 August 2021, Boulder, Colorado, {USA}},
  volume 134 of {\em Proceedings of Machine Learning Research}, pages 93--158.
  {PMLR}, 2021.

\bibitem[ADM{\etalchar{+}}10]{ADMOP10}
J.~Acharya, H.~Das, H.~Mohimani, A.~Orlitsky, and S.~Pan.
\newblock Exact calculation of pattern probabilities.
\newblock In {\em 2010 IEEE International Symposium on Information Theory},
  pages 1498--1502, June 2010.

\bibitem[ADOS17]{ADOS16}
Jayadev Acharya, Hirakendu Das, Alon Orlitsky, and Ananda~Theertha Suresh.
\newblock A unified maximum likelihood approach for estimating symmetric
  properties of discrete distributions.
\newblock In Doina Precup and Yee~Whye Teh, editors, {\em Proceedings of the
  34th International Conference on Machine Learning}, volume~70 of {\em
  Proceedings of Machine Learning Research}, pages 11--21. PMLR, 06--11 Aug
  2017.

\bibitem[AOST14]{AOST14}
Jayadev Acharya, Alon Orlitsky, Ananda~Theertha Suresh, and Himanshu Tyagi.
\newblock The complexity of estimating rényi entropy.
\newblock In {\em Proceedings of the Twenty-Sixth Annual ACM-SIAM Symposium on
  Discrete Algorithms}, pages 1855--1869, 2014.

\bibitem[AOST17]{AOST17}
Jayadev Acharya, Alon Orlitsky, Ananda~Theertha Suresh, and Himanshu Tyagi.
\newblock Estimating renyi entropy of discrete distributions.
\newblock {\em IEEE Trans. Inf. Theor.}, 63(1):38--56, January 2017.

\bibitem[AW21]{AlmanW21}
Josh Alman and Virginia~Vassilevska Williams.
\newblock A refined laser method and faster matrix multiplication.
\newblock In {\em Proceedings of the 2021 {ACM-SIAM} Symposium on Discrete
  Algorithms, {SODA} 2021, Virtual Conference, January 10 - 13, 2021}, pages
  522--539. {SIAM}, 2021.

\bibitem[BF93]{BF93}
John Bunge and Michael Fitzpatrick.
\newblock Estimating the number of species: a review.
\newblock {\em Journal of the American Statistical Association},
  88(421):364--373, 1993.

\bibitem[BZLV16]{BZLV16}
Y.~{Bu}, S.~{Zou}, Y.~{Liang}, and V.~V. {Veeravalli}.
\newblock Estimation of kl divergence between large-alphabet distributions.
\newblock In {\em 2016 IEEE International Symposium on Information Theory
  (ISIT)}, pages 1118--1122, July 2016.

\bibitem[CCG{\etalchar{+}}12]{CCGLMCL12}
Robert~K Colwell, Anne Chao, Nicholas~J Gotelli, Shang-Yi Lin, Chang~Xuan Mao,
  Robin~L Chazdon, and John~T Longino.
\newblock Models and estimators linking individual-based and sample-based
  rarefaction, extrapolation and comparison of assemblages.
\newblock {\em Journal of plant ecology}, 5(1):3--21, 2012.

\bibitem[Cha84]{Chao84}
A~Chao.
\newblock Nonparametric estimation of the number of classes in a population.
  scandinavianjournal of statistics11, 265-270.
\newblock {\em Chao26511Scandinavian Journal of Statistics1984}, 1984.

\bibitem[CSS19a]{CSS19}
Moses Charikar, Kirankumar Shiragur, and Aaron Sidford.
\newblock Efficient profile maximum likelihood for universal symmetric property
  estimation.
\newblock In {\em Proceedings of the 51st Annual ACM SIGACT Symposium on Theory
  of Computing}, STOC 2019, pages 780--791, New York, NY, USA, 2019. ACM.

\bibitem[CSS19b]{CSS19pseudo}
Moses Charikar, Kirankumar Shiragur, and Aaron Sidford.
\newblock A general framework for symmetric property estimation.
\newblock In H.~Wallach, H.~Larochelle, A.~Beygelzimer, F.~d~Alch\'{e}-Buc,
  E.~Fox, and R.~Garnett, editors, {\em Advances in Neural Information
  Processing Systems 32}, pages 12447--12457. Curran Associates, Inc., 2019.

\bibitem[DS13]{DS13}
Timothy Daley and Andrew~D Smith.
\newblock Predicting the molecular complexity of sequencing libraries.
\newblock {\em Nature methods}, 10(4):325, 2013.

\bibitem[ET76]{ET76}
Bradley Efron and Ronald Thisted.
\newblock Estimating the number of unseen species: How many words did
  shakespeare know?
\newblock {\em Biometrika}, 63(3):435--447, 1976.

\bibitem[F{\"u}r05]{Fur05}
Johannes F{\"u}rnkranz.
\newblock Web mining.
\newblock In {\em Data mining and knowledge discovery handbook}, pages
  899--920. Springer, 2005.

\bibitem[Gal14]{Gall14a}
Fran{\c{c}}ois~Le Gall.
\newblock Powers of tensors and fast matrix multiplication.
\newblock In {\em International Symposium on Symbolic and Algebraic
  Computation, {ISSAC} '14, Kobe, Japan, July 23-25, 2014}, pages 296--303.
  {ACM}, 2014.

\bibitem[GTPB07]{GTPB07}
Zhan Gao, Chi-hong Tseng, Zhiheng Pei, and Martin~J Blaser.
\newblock Molecular analysis of human forearm superficial skin bacterial biota.
\newblock {\em Proceedings of the National Academy of Sciences},
  104(8):2927--2932, 2007.

\bibitem[Han21]{Han20}
Yanjun Han.
\newblock On the high accuracy limitation of adaptive property estimation.
\newblock In Arindam Banerjee and Kenji Fukumizu, editors, {\em Proceedings of
  The 24th International Conference on Artificial Intelligence and Statistics},
  volume 130 of {\em Proceedings of Machine Learning Research}, pages 910--918.
  PMLR, 13--15 Apr 2021.

\bibitem[HHRB01]{HHRB01}
Jennifer~B Hughes, Jessica~J Hellmann, Taylor~H Ricketts, and Brendan~JM
  Bohannan.
\newblock Counting the uncountable: statistical approaches to estimating
  microbial diversity.
\newblock {\em Appl. Environ. Microbiol.}, 67(10):4399--4406, 2001.

\bibitem[HJW16]{HJW16}
Yanjun Han, Jiantao Jiao, and Tsachy Weissman.
\newblock Minimax estimation of {KL} divergence between discrete distributions.
\newblock {\em ArXiv e-prints}, abs/1605.09124, 2016.

\bibitem[HJW18]{HJW18}
Yanjun Han, Jiantao Jiao, and Tsachy Weissman.
\newblock Local moment matching: A unified methodology for symmetric functional
  estimation and distribution estimation under wasserstein distance.
\newblock {\em ArXiv e-prints, arXiv:1802.08405}, 2018.

\bibitem[HO19]{HO19}
Yi~{Hao} and Alon {Orlitsky}.
\newblock {The Broad Optimality of Profile Maximum Likelihood}.
\newblock {\em ArXiv e-prints}, page arXiv:1906.03794, Jun 2019.

\bibitem[HO20]{HO20pentropy}
Yi~Hao and Alon Orlitsky.
\newblock Profile entropy: A fundamental measure for the learnability and
  compressibility of distributions.
\newblock In {\em Proceedings of the 34th International Conference on Neural
  Information Processing Systems}, NIPS'20, Red Hook, NY, USA, 2020. Curran
  Associates Inc.

\bibitem[HS21]{HS20}
Yanjun Han and Kirankumar Shiragur.
\newblock On the competitive analysis and high accuracy optimality of profile
  maximum likelihood.
\newblock In D{\'{a}}niel Marx, editor, {\em Proceedings of the 2021 {ACM-SIAM}
  Symposium on Discrete Algorithms, {SODA} 2021, Virtual Conference, January 10
  - 13, 2021}, pages 1317--1336. {SIAM}, 2021.

\bibitem[JHW16]{JHW16}
J.~Jiao, Y.~Han, and T.~Weissman.
\newblock Minimax estimation of the l1 distance.
\newblock In {\em 2016 IEEE International Symposium on Information Theory
  (ISIT)}, pages 750--754, July 2016.

\bibitem[JVHW15]{JVHW15}
J.~Jiao, K.~Venkat, Y.~Han, and T.~Weissman.
\newblock Minimax estimation of functionals of discrete distributions.
\newblock {\em IEEE Transactions on Information Theory}, 61(5):2835--2885, May
  2015.

\bibitem[KLR99]{KLR99}
Ian Kroes, Paul~W Lepp, and David~A Relman.
\newblock Bacterial diversity within the human subgingival crevice.
\newblock {\em Proceedings of the National Academy of Sciences},
  96(25):14547--14552, 1999.

\bibitem[OSS{\etalchar{+}}04]{OSSVZ04}
A.~Orlitsky, S.~Sajama, N.~P. Santhanam, K.~Viswanathan, and Junan Zhang.
\newblock Algorithms for modeling distributions over large alphabets.
\newblock In {\em International Symposium on Information Theory, 2004. ISIT
  2004. Proceedings.}, pages 304--304, 2004.

\bibitem[OSW16]{OSW16}
Alon Orlitsky, Ananda~Theertha Suresh, and Yihong Wu.
\newblock Optimal prediction of the number of unseen species.
\newblock {\em Proceedings of the National Academy of Sciences},
  113(47):13283--13288, 2016.

\bibitem[PBG{\etalchar{+}}01]{PBGELLSD01}
Bruce~J Paster, Susan~K Boches, Jamie~L Galvin, Rebecca~E Ericson, Carol~N Lau,
  Valerie~A Levanos, Ashish Sahasrabudhe, and Floyd~E Dewhirst.
\newblock Bacterial diversity in human subgingival plaque.
\newblock {\em Journal of bacteriology}, 183(12):3770--3783, 2001.

\bibitem[PJW17]{PJW17}
D.~S. {Pavlichin}, J.~{Jiao}, and T.~{Weissman}.
\newblock {Approximate Profile Maximum Likelihood}.
\newblock {\em ArXiv e-prints, arXiv:1712.07177}, December 2017.

\bibitem[RCS{\etalchar{+}}09]{RCSWTKRWC09}
Harlan~S Robins, Paulo~V Campregher, Santosh~K Srivastava, Abigail Wacher,
  Cameron~J Turtle, Orsalem Kahsai, Stanley~R Riddell, Edus~H Warren, and
  Christopher~S Carlson.
\newblock Comprehensive assessment of t-cell receptor $\beta$-chain diversity
  in $\alpha$$\beta$ t cells.
\newblock {\em Blood}, 114(19):4099--4107, 2009.

\bibitem[TE87]{TE87}
Ronald Thisted and Bradley Efron.
\newblock Did shakespeare write a newly-discovered poem?
\newblock {\em Biometrika}, 74(3):445--455, 1987.

\bibitem[Von12]{Von12}
Pascal~O. Vontobel.
\newblock The bethe approximation of the pattern maximum likelihood
  distribution.
\newblock In {\em 2012 IEEE International Symposium on Information Theory
  Proceedings}, pages 2012--2016, 2012.

\bibitem[Von14]{Von14}
P.~O. Vontobel.
\newblock The bethe and sinkhorn approximations of the pattern maximum
  likelihood estimate and their connections to the valiant-valiant estimate.
\newblock In {\em 2014 Information Theory and Applications Workshop (ITA)},
  pages 1--10, Feb 2014.

\bibitem[VV11a]{VV11b}
G.~Valiant and P.~Valiant.
\newblock The power of linear estimators.
\newblock In {\em 2011 IEEE 52nd Annual Symposium on Foundations of Computer
  Science}, pages 403--412, Oct 2011.

\bibitem[VV11b]{VV11a}
Gregory Valiant and Paul Valiant.
\newblock Estimating the unseen: An n/log(n)-sample estimator for entropy and
  support size, shown optimal via new clts.
\newblock In {\em Proceedings of the Forty-third Annual ACM Symposium on Theory
  of Computing}, STOC '11, pages 685--694, New York, NY, USA, 2011. ACM.

\bibitem[Wil12]{will12}
Virginia~Vassilevska Williams.
\newblock Multiplying matrices faster than coppersmith-winograd.
\newblock In {\em Proceedings of the Forty-Fourth Annual ACM Symposium on
  Theory of Computing}, STOC '12, page 887–898, New York, NY, USA, 2012.
  Association for Computing Machinery.

\bibitem[WY15]{WY15}
Y.~{Wu} and P.~{Yang}.
\newblock {Chebyshev polynomials, moment matching, and optimal estimation of
  the unseen}.
\newblock {\em ArXiv e-prints, arXiv:1504.01227}, April 2015.

\bibitem[WY16]{WY16}
Y.~Wu and P.~Yang.
\newblock Minimax rates of entropy estimation on large alphabets via best
  polynomial approximation.
\newblock {\em IEEE Transactions on Information Theory}, 62(6):3702--3720, June
  2016.

\bibitem[ZVV{\etalchar{+}}16]{ZVVKCSLSDM16}
James Zou, Gregory Valiant, Paul Valiant, Konrad Karczewski, Siu~On Chan,
  Kaitlin Samocha, Monkol Lek, Shamil Sunyaev, Mark Daly, and Daniel~G.
  MacArthur.
\newblock Quantifying unobserved protein-coding variants in human populations
  provides a roadmap for large-scale sequencing projects.
\newblock {\em Nature Communications}, 7:13293 EP--, Oct 2016.

\end{thebibliography}
\bibliographystyle{alpha}
\newpage
\newpage
\appendix
\section{Proof of \Cref{thm:statmain}}
Here we provide the proof of \Cref{thm:statmain}. The proof of this statement is implicit in the analysis presented in \cite{HS20}. Here we provide a simpler and short proof that uses the continuity lemma presented in \cite{HS20}. For convenience, we restate this key lemma in our notation.
\begin{lem}[Lemma 2 in \cite{HS20}]\label{lem:cont}
Let $A \geq 2$,$c_0 \in (0,1), r,s$ be arbitrary constants with $0 < s < r \leq 1/2$ and let $\bR=\bR_{n,r}$. Then there exists a constant $c = c(A,c_0,r,s) > 0$ such that for any distribution $\bp \in \simplex$, there exists a pseudo distribution $\bq \in \dsimplex$ such that: for all $S \subseteq \Phi^n$, it holds that,
\begin{align*}
	\prob{\bp,S} &\geq \prob{\frac{\bq}{\|\bq\|_1},S}^{1/(1-c_on^{-s})} \exp(-cn^{1-2r+s}) \\
	\prob{\frac{\bq}{\|\bq\|_1},S} &\geq \prob{\bp,S}^{1/(1-c_on^{-s})} \exp(-cn^{1-2r+s})
\end{align*}
\end{lem}

\begin{proof}[Proof of \Cref{thm:statmain}]
Let $\bp$ be the hidden distribution. Given a profile $\phi$, let $\bq_{\phi} \in \dsimplex$ be any pseudo distribution that satisfies,
\begin{align*}
\prob{\frac{\bq_{\phi}}{\|\bq_{\phi}\|_1},\phi} & \geq \exp(-O(|\bR| \log n))\max_{\bq \in \simplex_{\bR}} \prob{\frac{\bq}{\|\bq\|_1},\phi}\\
&\geq  \exp(-O(n^{-1/3} \log^2 n))\max_{\bq \in \simplex_{\bR}} \prob{\frac{\bq}{\|\bq\|_1},\phi}~.
\end{align*}
As $L\defeq|\dsimplex| \leq \exp(O(n^{1/3}\log^2 n))$, we use $\bq_1,\dots \bq_{L}$ to denote the pseudo distributions in $\dsimplex$. Let $G\defeq \{\phi \in \Phi^n~|~|f(\bp)-\widehat{f}(\phi)| \leq \epsilon \}$, that is, the set of all profiles where the estimator succeeds. Also let $S_{i}=\{\phi \in G~|~\bq_{\phi}=\bq_i \}$. Using these definitions in the remainder of the proof, we upper bound the failure probability of our estimator. 
\begin{align*}
    \prob{\left|f(\frac{\bq_{\phi}}{\|\bq_{\phi}\|_1})-f(\bp)\right|>\epsilon}& \leq \prob{\bp,\Phi^n \backslash G} + \sum_{\{i\in [1,|L|]~|~ |f(\frac{\bq_i}{\|\bq_i\|_1})-f(\bp)| > \epsilon\}} \prob{\bp,S_{i}}~,\\
    &\leq \delta+ \sum_{\{i\in [1,|L|]~|~ |f(\frac{\bq_i}{\|\bq_i\|_1})-f(\bp)| > \epsilon\}} \prob{\bp,S_{i}}~.
\end{align*}
In the above inequality we used that the failure probability of the estimator is at most $\delta$, that is, $\prob{\bp,\Phi^n \backslash G} \leq \delta$.
First note that from the definitions of $\bq_i$, $S_i$ and $\bq_{\phi}$, we have that,
$$\prob{\frac{\bq_i}{\|\bq_i\|_1},S_i} \geq \exp(-O(n^{-1/3} \log^2 n))\max_{\bq \in \simplex_{\bR}} \prob{\frac{\bq}{\|\bq\|_1},S_i}~.$$
Further applying \Cref{lem:cont} with $r=1/3$ and $s$ with any tiny constant, we get that,
$$\prob{\frac{\bq_i}{\|\bq_i\|_1},S_i} \geq \exp(-O(n^{-1/3} \log n)) \prob{{\bp},S_i}^{1+O(n^{-c'})} \exp(-O(n^{1/3+c'}))~,$$
for any tiny constant $c'>0$. The above expression further simplifies to,
$$\prob{\frac{\bq_i}{\|\bq_i\|_1},S_i} \geq \exp(-O(n^{1/3+c'})) \prob{{\bp},S_i}^{1+O(n^{-c'})} ~.$$
Suppose $S_i$ is set such that, $\prob{\bp,S_{i}} \geq \delta^{\frac{1}{1+O(n^{-c'})}} \exp(\frac{O(n^{1/3+c'})}{1+O(n^{-c'})})$, then note that $\prob{\frac{\bq_i}{\|\bq_i\|_1},S_i} >\delta$. As the estimator provided by the conditions of the lemma succeeds on all the profiles in $S_i$, we that $|f(\bp)-\widehat{f}(\phi)| \leq \epsilon$ for all $\phi \in S_i$. Suppose $|f(\bp)-f(\frac{\bq_i}{\|\bq_i\|_1})|>2\epsilon$, then by triangle inequality this would imply $|f(\frac{\bq_i}{\|\bq_i\|_1})-\widehat{f}(\phi)| > \epsilon$ for all $\phi \in S_i$. However note that, $\prob{\frac{\bq_i}{\|\bq_i\|_1},S_i}>\delta$, and this would imply that the failure probability of the estimator is greater than $\delta$ when the underlying distribution is $\frac{\bq_i}{\|\bq_i\|_1}$; a contradiction. Therefore, it should be the case that $|f(\bp)-f(\frac{\bq_i}{\|\bq_i\|_1})|\leq 2\epsilon$ for all $S_i$ that satisfy $\prob{\bp,S_{i}} \geq \delta^{\frac{1}{1+O(n^{-c'})}} \exp(\frac{O(n^{1/3+c'})}{1+O(n^{-c'})})$ and our failure probability is upper bounded by,
\begin{align*}
    \prob{\left|f(\frac{\bq_{\phi}}{\|\bq_{\phi}\|_1})-f(\bp)\right|>\epsilon}& \leq \delta+ |L|\delta^{1-O(n^{-c'})} \exp(O(n^{1/3+c'}))~.
\end{align*}
Further, substituting the value of $|L| \leq \exp(O(n^{1/3} \log^2 n))$, we get our desired result and we conclude our proof.
\end{proof}

\section{Other results}\label{app:other}
\begin{lem}
For any two vectors $u ,v \in \R_{\geq 0}^{[0,k]}$, the following inequality holds,
$$(u^{T}\onevec \log u^{T}\onevec + v^{T}\onevec \log v^{T}\onevec)-\sum_{i \in [0,k]}(u_i \log u_i+v_i \log v_i) \leq w^{T}\onevec \log w^{T}\onevec-\sum_{i \in [0,k]}w_i \log w_i$$
where $w=u+v$.
\end{lem}
\begin{proof}
For any $x \in \R_{\geq 0}^{[0,k]}$, let $f(x)\defeq x^{T}\onevec \log x^{T}\onevec-\sum_{i \in [0,k]}x_i \log x_i$. Note that $f(x)$ is concave~\cite{CSS19} and furthermore, $f(c\cdot x)=c\cdot f(x)$.

Let $w'=\frac{1}{2}u+\frac{1}{2}v$, applying concavity we get that,
$$ f(w') \geq \frac{1}{2}f(u)+\frac{1}{2}f(v),$$
As $f(w)=2f(w')$, combined with above inequality we have our proof.
\end{proof}

\begin{lem}\label{lem:lips}
    For many matrices $\mx,\my \in \R^{\ell \times [0,k]}$ (where $\|\mx\|_1,\|\my\|_1,k,\ell \leq O(n^2)$) such that $\|\mx-\my\|_1 \leq \alpha$, we have that,
    $$\left|\sum_{i \in \ell, j\in [0,k]}\mx_{ij}\log \mx_{ij}-\sum_{i \in \ell, j\in [0,k]}\my_{ij}\log \my_{ij}\right| \leq O(\alpha \log n) + O(\log n)~.$$
    Furthermore,
    $$\left| \sum_{i \in \ell} [\mx \onevec]_{i} \log[\mx \onevec]_{i} -\sum_{i \in \ell} [\my \onevec]_{i} \log[\my \onevec]_{i} \right| \leq O(\alpha \log n) + O(\log n)~.$$
\end{lem}
\begin{proof}

Note that $x\log x$ is $O(\log n)$-Lipshcitz for $x \geq \frac{1}{n^2}$, and for any $0\leq x_1,x_2\leq \frac{1}{n^2}$, we have $|x_1\log x_1-x_2\log x_2|\leq O(\frac{1}{n^2}\log n)$. As a result, for any non-negative $x_1,x_2$ that are $O(n^2)$, we have 
\begin{align*}
    |x_1\log x_1 -x_2\log x_2|\leq
\left(\frac{1}{n^2}+|x_1-x_2|\right)\cdot O(\log n)
\end{align*}
Note that $\sum_{i \in \ell, j\in [0,k]}\mx_{ij}\leq O(n^2)$ and $\sum_{i \in \ell, j\in [0,k]}\my_{ij}\leq O(n^2)$. We have
\begin{align*}
    &\left|\sum_{i \in \ell, j\in [0,k]}\mx_{ij}\log \mx_{ij}-\sum_{i \in \ell, j\in [0,k]}\my_{ij}\log \my_{ij}\right|
     \leq \sum_{i \in \ell, j\in [0,k]}|\mx_{ij}\log \mx_{ij}-\my_{ij}\log \my_{ij}|\\
     & \leq \sum_{i \in \ell, j\in [0,k]}(\frac{1}{n^2}+|\mx_{ij}-\my_{ij}|) O(\log n)  
     \leq \left(1+\|\mx-\my\|_1\right)\cdot O(\log n)
     \leq O(\alpha \log n)+O(\log n)
\end{align*}
and similarly,
\begin{align*}
    &\left|\sum_{i \in \ell} [\mx \onevec]_{i} \log[\mx \onevec]_{i} -\sum_{i \in \ell} [\my \onevec]_{i} \log[\my \onevec]_{i}\right| 
     \leq \sum_{i \in \ell}\left| [\mx \onevec]_{i} \log[\mx \onevec]_{i} - [\my \onevec]_{i} \log[\my \onevec]_{i}\right|\\
     &\leq \sum_{i \in \ell}(\frac{1}{n^2}+|[\mx \onevec]_{i}-[\my \onevec]_{i}|) O(\log n) 
     \leq (\frac{1}{n}+\|\mx-\my\|_1) O(\log n)
     \leq O(\alpha \log n)+O(\log n)
\end{align*}
We conclude the proof.
\end{proof}

\section{Description and guarantees of $\pround$}\label{subsec:partial}
Here we provide the description of $\pround$ and its guarantees are summarized in \Cref{lem:partial-round}.

\begin{algorithm}[H]
	\caption{$\pround(\mx,i)$}
	\begin{algorithmic}[1]
	\State $\mtz=\msplit{\mx}{i}$. \;
	\State $(\mw,j)=\proundx(\mtz,i+1)$. \;
	\State $\my=\mcombine{\mw}{i}$. \;
	\State Return $(\my,j)$. \;
	\end{algorithmic}
	\label{alg:pround}
\end{algorithm}


\lempround*

Our algorithm invokes several subroutines and in the following we provide description for each of these and also state their guarantees. The subroutine $\proundx$ provides an algorithm that proves the \Cref{lem:partial-round} in the special case of $[\mx \onevec]_i \in [0,1)$. The description of this subroutine and its guarantees are stated below. For convenience, we first introduce some new notation and a new operation $\transname$.
\paragraph{Notation:} For any matrix $\mx \in \R_{\geq 0}^{\ell \times [0,k]}$ and indices $i_1,i_2,j_1,j_2$, we define $\summx{i_1}{i_2}{j_1}{j_2} \defeq \sum_{i=i_1}^{i_2}\sum_{j=j_1}^{j_2}\mx_{ij}$ and say $(i,j)\in \twodinterval{i_1}{i_2}{j_1}{j_2}$ if $i_1\leq i\leq i_2$ and $j_1\leq j\leq j_2$. Therefore, $\summx{i_1}{i_2}{j_1}{j_2} = \sum_{(i,j)\in \twodinterval{i_1}{i_2}{j_1}{j_2}}\mx_{ij}$. Also note that when $i_1\geq i_2+1$ or $j_1\geq j_2+1$, we are summing over a empty set and $\summx{i_1}{i_2}{j_1}{j_2}\defeq 0$.

\begin{lem}[Guarantees of $\transname$]\label{lem:trans}
Let $\mx \in \R_{\geq 0}^{\ell \times [0,k]}$. For any indices $i_1\leq i_2+1 \leq i_3\leq i_4+1$, $j_1\leq j_2+1 \leq j_3\leq j_4+1$, and $0\leq v\leq \min\{  \summx{i_3}{i_4}{j_1}{j_2}, \summx{i_1}{i_2}{j_3}{j_4}\}$, there exists $\my \in \R_{\geq 0}^{\ell \times [0,k]}$ such that,
$\my$ is within $v$-swap distance from $\mx$, and
\begin{align*}
  &\bullet ~\summy{i_1}{i_2}{j_1}{j_2} = \summx{i_1}{i_2}{j_1}{j_2}+v ~~ \bullet ~\summy{i_3}{i_4}{j_3}{j_4} = \summx{i_3}{i_4}{j_3}{j_4}+v,\\
  &\bullet ~\summy{i_1}{i_2}{j_3}{j_4} = \summx{i_1}{i_2}{j_3}{j_4}-v ~~ \bullet ~\summy{i_3}{i_4}{j_1}{j_2} = \summx{i_3}{i_4}{j_1}{j_2}-v.\\
  &\bullet \my_{ij}=\mx_{ij} ~~\forall (i,j)\notin \twodinterval{i_1}{i_2}{j_1}{j_2} \cup \twodinterval{i_1}{i_2}{j_3}{j_4} \cup \twodinterval{i_3}{i_4}{j_1}{j_2} \cup \twodinterval{i_3}{i_4}{j_3}{j_4}.
\end{align*}

  
Furthermore, we define operation $\trans{\mx}{v}{\transpara{i_1}{i_2}{i_3}{i_4}{j_1}{j_2}{j_3}{j_4}}$ whose output is $\my$ that satisfies properties above, and the running time of process $\transname$ is $O(\ell k)$.
\end{lem}

\figurepack{}{$\transname$~operation. It increases two blocks by $v$ respectively, and decreases two blocks by $v$ respectively.}{
\begin{tikzpicture} 
\draw (0,0) rectangle (5,5);
\draw [help lines] (0,1) -- (5,1);
\draw [help lines] (0,2) -- (5,2);
\draw [help lines] (0,3) -- (5,3);
\draw [help lines] (0,4) -- (5,4);
\draw [help lines] (1,0) -- (1,5);
\draw [help lines] (2,0) -- (2,5);
\draw [help lines] (3,0) -- (3,5);
\draw [help lines] (4,0) -- (4,5);
\node  at (-0.2,4.8) {$1$} ;
\node  at (-0.2,3.8) {$i_1$} ;
\node  at (-0.2,3.2) {$i_2$} ;
\node  at (-0.2,1.8) {$i_3$} ;
\node  at (-0.2,1.2) {$i_4$} ;
\node  at (-0.2,0.2) {$\ell$} ;
\node  at (0.2,5.2) {$0$} ;
\node  at (1.2,5.2) {$j_1$} ;
\node  at (1.8,5.2) {$j_2$} ;
\node  at (3.2,5.2) {$j_3$} ;
\node  at (3.8,5.2) {$j_4$} ;
\node  at (4.8,5.2) {$k$} ;
\node  at (1.5,3.5) {$+v$} ;
\node  at (3.5,3.5) {$-v$} ;
\node  at (1.5,1.5) {$-v$} ;
\node  at (3.5,1.5) {$+v$} ;
\end{tikzpicture}
}

We defer the proof for \cref{lem:trans} to \cref{sec:proof:trans}. The description of our subroutine $\proundx$ is as follows.
\begin{algorithm}[H]
	\caption{$\proundx(\mx,i)$}
	\begin{algorithmic}[1]
	\State $j=\min\{r:  \mx(1:\ell, 0:r) > \mx(1:i-1, 0:k)$. \;
	\State $t=\mx(1:i-1, 0:k) -\mx(1:\ell, 0:j-1)$.\;
    \State $\my=\mx$.\;
	\State $\my=\trans{\my}{\summy{i}{i}{0}{j-1}}{\transpara{1}{i-1}{i}{i}{0}{j-1}{j}{k}}$. \label{alg:prs:1}\;
	\If{ $\summy{1}{i-1}{j}{j}\geq t$ }
	    \State $v=\summy{1}{i-1}{j}{j}-\left\lfloor\summy{1}{i-1}{j}{j}\right\rfloor$. \;
	    \State $\my=\trans{\my}{\min\left\{v+1 , \summy{1}{i-1}{j}{j}-t\right\}}{\transpara{1}{i-1}{i+1}{\ell}{0}{j-1}{j}{j}}$. \label{alg:prs:2} \;
	\Else
	    \State $v=\left\lceil\summy{1}{i-1}{j}{j}\right\rceil - \summy{1}{i-1}{j}{j}$. \;
	    \State $\my=\trans{\my}{v}{\transpara{1}{i-1}{i}{\ell}{j}{j}{j+1}{k}}$. \label{alg:prs:3} \;
	   \EndIf
	\State $\my=\trans{\my}{\summy{i}{i}{j+1}{k}}{\transpara{i}{i}{i+1}{\ell}{j}{j}{j+1}{k}}$. \label{alg:prs:4} \;
	\State Return $(\my,j)$. \label{alg:prs:5} \;
	\end{algorithmic}
	\label{alg:proundx}
\end{algorithm}

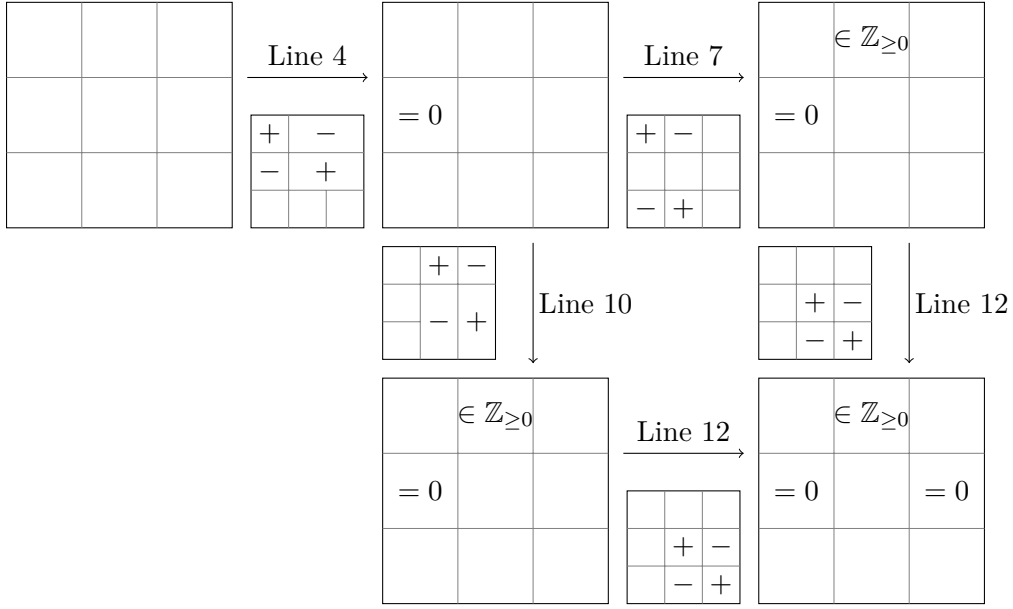
\figurepack{}{Algorithm \ref{alg:proundx}. We use a $3\times 3$ partitioned matrix to denote $\my\in \R_{\geq 0}^{\ell \times [0,k]}$, where the upper (middle / lower) row denotes the first $(i-1)$ rows (the $i$-th row / the last $(\ell-i)$ rows) of $\my$, and the left (middle / right) column denotes the $0$-th to $(j-1)$-th columns (the $j$-th column / the last $(k-j)$ columns) of $\my$. Symbol $+$ or $-$ means increasing or decreasing some elements in the block. Symbol $=0$ means all the elements in the block are zero. Symbol $\in \Z_{\geq 0}$ means the sum of the block is a non-negative integer.}{
\begin{tikzpicture} 
\coordinate (p1) at (0,0);
\coordinate (p2) at ([xshift=5cm] p1);
\coordinate (p3) at ([xshift=5cm] p2);
\coordinate (p33) at ([yshift=-5cm] p2);
\coordinate (p4) at ([yshift=-5cm] p3);
\coordinate (o1) at ([xshift=3cm,yshift=3cm] p1);
\coordinate (o2) at ([xshift=3cm,yshift=3cm] p2);
\coordinate (o3) at ([xshift=3cm,yshift=3cm] p3);
\coordinate (o33) at ([xshift=3cm,yshift=3cm] p33);
\coordinate (o4) at ([xshift=3cm,yshift=3cm] p4);
\coordinate (q1) at ([xshift=3.25cm] p1);
\coordinate (q2) at ([xshift=3.25cm] p2);
\coordinate (q22) at ([yshift=-1.75cm] p2);
\coordinate (q3) at ([yshift=-1.75cm] p3);
\coordinate (q33) at ([xshift=3.25cm] p33);
\coordinate (q4) at ([xshift=3.25cm] p2);
\coordinate (r1) at ([xshift=1.5cm,yshift=1.5cm] q1);
\coordinate (r2) at ([xshift=1.5cm,yshift=1.5cm] q2);
\coordinate (r22) at ([xshift=1.5cm,yshift=1.5cm] q22);
\coordinate (r3) at ([xshift=1.5cm,yshift=1.5cm] q3);
\coordinate (r33) at ([xshift=1.5cm,yshift=1.5cm] q33);
\draw (p1) rectangle (o1);
\draw (p2) rectangle (o2);
\draw (p3) rectangle (o3);
\draw (p33) rectangle (o33);
\draw (p4) rectangle (o4);
\draw (q1) rectangle (r1);
\draw (q2) rectangle (r2);
\draw (q22) rectangle (r22);
\draw (q3) rectangle (r3);
\draw (q33) rectangle (r33);
\draw [help lines, shift=(p1)]
{
(0,1)--(3,1) 
(0,2)--(3,2)
(1,0)--(1,3)
(2,0)--(2,3)
};
\draw [help lines, shift=(p2)]
{
(0,1)--(3,1) 
(0,2)--(3,2)
(1,0)--(1,3)
(2,0)--(2,3)
};
\draw [help lines, shift=(p3)]
{
(0,1)--(3,1) 
(0,2)--(3,2)
(1,0)--(1,3)
(2,0)--(2,3)
};
\draw [help lines, shift=(p33)]
{
(0,1)--(3,1) 
(0,2)--(3,2)
(1,0)--(1,3)
(2,0)--(2,3)
};
\draw [help lines, shift=(p4)]
{
(0,1)--(3,1) 
(0,2)--(3,2)
(1,0)--(1,3)
(2,0)--(2,3)
};
\draw[->, shift=(p1)]  (3.2,2) --  (4.8,2);
\draw[->, shift=(p2)]  (3.2,2) --  (4.8,2);
\draw[->, shift=(p33)]  (3.2,2) --  (4.8,2);
\draw[->, shift=(p2)]  (2,-0.2) --  (2,-1.8);
\draw[->, shift=(p3)]  (2,-0.2) --  (2,-1.8);
\node at ([xshift=4cm,yshift=2.3cm] p1)  {  \Cref{alg:prs:1}};
\node at ([xshift=4cm,yshift=2.3cm] p2)  {  \Cref{alg:prs:2}};
\node at ([xshift=4cm,yshift=2.3cm] p33)  {  \Cref{alg:prs:4}};
\node at ([xshift=2.7cm,yshift=-1cm] p2)  {  \Cref{alg:prs:3}};
\node at ([xshift=2.7cm,yshift=-1cm] p3)  {  \Cref{alg:prs:4}};
\draw [help lines, shift=(q1)]
{
(0,0.5)--(1.5,0.5) 
(0,1)--(1.5,1)
(0.5,0)--(0.5,1.5)
(1,0)--(1,0.5)
};
\node at ([xshift=0.25cm,yshift=1.25cm] q1) {$+$};
\node at ([xshift=0.25cm,yshift=0.75cm] q1) {$-$};
\node at ([xshift=1cm,yshift=1.25cm] q1) {$-$};
\node at ([xshift=1cm,yshift=0.75cm] q1) {$+$};
\draw [help lines, shift=(q2)]
{
(0,0.5)--(1.5,0.5) 
(0,1)--(1.5,1)
(0.5,0)--(0.5,1.5)
(1,0)--(1,1.5)
};
\node at ([xshift=0.25cm,yshift=1.25cm] q2) {$+$};
\node at ([xshift=0.25cm,yshift=0.25cm] q2) {$-$};
\node at ([xshift=0.75cm,yshift=1.25cm] q2) {$-$};
\node at ([xshift=0.75cm,yshift=0.25cm] q2) {$+$};
\draw [help lines, shift=(q22)]
{
(0,0.5)--(0.5,0.5) 
(0,1)--(1.5,1)
(0.5,0)--(0.5,1.5)
(1,0)--(1,1.5)
};
\node at ([xshift=0.75cm,yshift=1.25cm] q22) {$+$};
\node at ([xshift=0.75cm,yshift=0.5cm] q22) {$-$};
\node at ([xshift=1.25cm,yshift=1.25cm] q22) {$-$};
\node at ([xshift=1.25cm,yshift=0.5cm] q22) {$+$};
\draw [help lines, shift=(q3)]
{
(0,0.5)--(1.5,0.5) 
(0,1)--(1.5,1)
(0.5,0)--(0.5,1.5)
(1,0)--(1,1.5)
};
\node at ([xshift=0.75cm,yshift=0.75cm] q3) {$+$};
\node at ([xshift=0.75cm,yshift=0.25cm] q3) {$-$};
\node at ([xshift=1.25cm,yshift=0.75cm] q3) {$-$};
\node at ([xshift=1.25cm,yshift=0.25cm] q3) {$+$};
\draw [help lines, shift=(q33)]
{
(0,0.5)--(1.5,0.5) 
(0,1)--(1.5,1)
(0.5,0)--(0.5,1.5)
(1,0)--(1,1.5)
};
\node at ([xshift=0.75cm,yshift=0.75cm] q33) {$+$};
\node at ([xshift=0.75cm,yshift=0.25cm] q33) {$-$};
\node at ([xshift=1.25cm,yshift=0.75cm] q33) {$-$};
\node at ([xshift=1.25cm,yshift=0.25cm] q33) {$+$};
\node at ([xshift=0.5cm,yshift=1.5cm] p2) {$=0$};
\node at ([xshift=0.5cm,yshift=1.5cm] p3) {$=0$};
\node at ([xshift=0.5cm,yshift=1.5cm] p33) {$=0$};
\node at ([xshift=0.5cm,yshift=1.5cm] p4) {$=0$};
\node at ([xshift=1.5cm,yshift=2.5cm] p3) {$\in \Z_{\geq 0}$};
\node at ([xshift=1.5cm,yshift=2.5cm] p33) {$\in \Z_{\geq 0}$};
\node at ([xshift=1.5cm,yshift=2.5cm] p4) {$\in \Z_{\geq 0}$};
\node at ([xshift=2.5cm,yshift=1.5cm] p4) {$=0$};
%

\end{tikzpicture}
}


\begin{lem}[Guarantee of $\proundx$]\label{lem:proundx}
The algorithm $\proundx$ takes in inputs $\mx \in \R^{\ell \times [0,k]}_{\geq 0}$ and $i \in [1,\ell-1]$ that satisfy the conditions of \Cref{lem:partial-round} with an additions assumption that $\sum_{j' \in [0,k]}\mx_{ij'} \in [0,1)$ and outputs a matrix $\my \in \R_{\geq 0}^{\ell \times [0,k]}$ and an index $j$ that satisfy the guarantees of \Cref{lem:partial-round}. Furthermore, the running time of this procedure is $O(\ell k)$.
\end{lem}
We defer the proof for \cref{lem:proundx} to \cref{sec:proof:proundx}. To extend the above algorithm for the general case, we define simple operations $\msplitname$ and $\mcombinename$ that we define next. Intuitively, the operation $\mcombinename$ combines $i$-th and $(i+1)$-th rows by adding them up, and $\msplitname$ splits $i$-th row into two rows so that the summation of one row is an integer and the summation of the other row is less than one.
\begin{defn}[Combine]
For any $\mw \in \R_{\geq 0}^{(\ell+1) \times [0,k]}$ and $t\in [1,\ell]$, $\my\defeq \mcombine{\mw}{t}\in \R_{\geq 0}^{\ell \times [0,k]}$ is defined as follows,
\begin{align*}
    \my_{i,j}=
    \begin{cases}
    \mw_{i,j} & i\leq t-1    \\
    \mw_{i,j}+\mw_{(i+1),j} & i= t    \\
    \mw_{(i+1),j} & i\geq t+1 
    \end{cases}
\end{align*}
\end{defn}

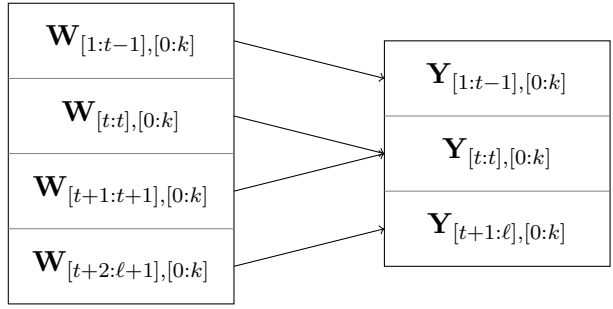
\figurepack{}{$\mcombinename$~operation. It combines the $t$-th row and the $(t+1)$-th row by adding them up}{
\begin{tikzpicture} 
\draw (0,0) rectangle (3,4);
\draw [help lines] (0,1) -- (3,1);
\draw [help lines] (0,2) -- (3,2);
\draw [help lines] (0,3) -- (3,3);
\node  at (1.5,3.5) {$\mw_{[1:t-1],[0:k]}$} ;
\node  at (1.5,2.5) {$\mw_{[t:t],[0:k]}$} ;
\node  at (1.5,1.5) {$\mw_{[t+1:t+1],[0:k]}$} ;
\node  at (1.5,0.5) {$\mw_{[t+2:\ell+1],[0:k]}$} ;
\draw (5,0.5) rectangle (8,3.5);
\draw [help lines] (5,1.5) -- (8,1.5);
\draw [help lines] (5,2.5) -- (8,2.5);
\node  at (6.5,3) {$\my_{[1:t-1],[0:k]}$} ;
\node  at (6.5,2) {$\my_{[t:t],[0:k]}$} ;
\node  at (6.5,1) {$\my_{[t+1:\ell],[0:k]}$} ;
\draw[->] (3,3.5)--(5,3);
\draw[->] (3,2.5)--(5,2);
\draw[->] (3,1.5)--(5,2);
\draw[->] (3,0.5)--(5,1);
\end{tikzpicture}
}

\begin{defn}[Split]
Let $\mx \in \R_{\geq 0}^{\ell \times [0,k]}$, $r(t):=\sum_{j\in [0,k]}\mx_{tj}$ and $s(t):=\min\{ j:\sum_{j'\leq j}\mx_{tj'}\geq \lfloor r(t)\rfloor  \}$.
We define $\msplit{\mx}{t}$ to be $\mtz\in \R_{\geq 0}^{(\ell+1) \times [0,k]}$ where,
\begin{align*}
    \mtz_{i,j}= \mx_{i,j}, ~~ i\leq t-1    \quad \quad \quad \quad \mtz_{i,j}= \mx_{(i-1),j},  i\geq t+2
\end{align*}
\begin{align*}
     \mtz_{t,j}=
    \begin{cases}
      \mx_{t,j}, & j<s(t) \\
        \lfloor r(t)\rfloor - \sum_{j'\leq j-1}\mx_{t,j'}, & j=s(t)\\
        0, & j>s(t)
    \end{cases}
    \quad \quad 
    \mtz_{t+1,j}=
    \begin{cases}
    0,  & j<s(t)    \\
     \sum_{j'\leq j}\mx_{t,j'} - \lfloor r(t)\rfloor, &  j=s(t)\\
    \mx_{t,j}, & j>s(t)
    \end{cases}
\end{align*}
 \end{defn}

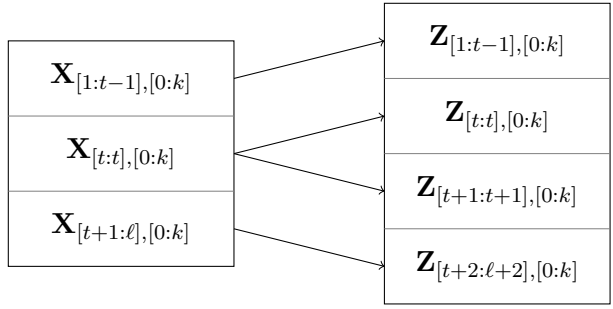
\figurepack{}{$\msplitname$~operation. It splits the $t$-th row into two rows, where the sum of one row is an integer, and the sum of the other row is less than one.}{
\begin{tikzpicture} 
\draw (0,0) rectangle (3,3);
\draw [help lines] (0,1) -- (3,1);
\draw [help lines] (0,2) -- (3,2);
\node  at (1.5,2.5) {$\mx_{[1:t-1],[0:k]}$} ;
\node  at (1.5,1.5) {$\mx_{[t:t],[0:k]}$} ;
\node  at (1.5,0.5) {$\mx_{[t+1:\ell],[0:k]}$} ;
\draw (5,-0.5) rectangle (8,3.5);
\draw [help lines] (5,0.5) -- (8,0.5);
\draw [help lines] (5,1.5) -- (8,1.5);
\draw [help lines] (5,2.5) -- (8,2.5);
\node  at (6.5,3) {$\mtz_{[1:t-1],[0:k]}$} ;
\node  at (6.5,2) {$\mtz_{[t:t],[0:k]}$} ;
\node  at (6.5,1) {$\mtz_{[t+1:t+1],[0:k]}$} ;
\node  at (6.5,0) {$\mtz_{[t+2:\ell+2],[0:k]}$} ;
\draw[->] (3,2.5)--(5,3);
\draw[->] (3,1.5)--(5,2);
\draw[->] (3,1.5)--(5,1);
\draw[->] (3,0.5)--(5,0);
\end{tikzpicture}
}

We remark that $\mtz$ has the following properties: 1) $\mcombine{\mtz}{t}=\mx$, 2) $\sum_{j\in [0,k]}\mtz_{tj}=\lfloor r(t) \rfloor$, 3) $\sum_{j\in [0,k]}\mtz_{(t+1),j}=r(t)-\lfloor r(t) \rfloor$, which is a real number in $[0,1)$.

\subsection{Proof for \cref{lem:trans}} \label{sec:proof:trans}

\begin{proof}[Proof for \Cref{lem:trans}]

We define $r(\mx)$ to be $\left\{(i,j)\in \twodinterval{i_3}{i_4}{j_1}{j_2} \bigg{|} \mx_{ij}>0 \right\}$, and $s(\mx)$ to be $\left\{(i,j)\in \twodinterval{i_1}{i_2}{j_3}{j_4} \bigg{|} \mx_{ij}>0 \right\}$. Note that $r(\mx)$ is also dependent on $i_3,i_4,j_1$ and $j_2$, but they are fixed in this proof, and $s(\mx)$ is also dependent on $i_1,i_2,j_3$ and $j_4$. We prove the lemma by induction.

When $|r(\mx)|+|s(\mx)|\leq 1$, one of $|r(\mx)|$ and $|s(\mx)|$ is 0, so $v=0$, and $\my=\mx$ is a feasible output.

Now we assume the lemma is true when $|r(\mx)|+|s(\mx)|<t$, and show it is still true when $|r(\mx)|+|s(\mx)|=t$, where $t\geq 2$.

If $|r(\mx)|=0$ or $|s(\mx)|=0$, we have $v=0$, and $\my=\mx$ is a feasible output.

If $|r(\mx)|>0$ and $|s(\mx)|>0$, we draw arbitrary elements $(c,b)\in r(\mx)$ and $(a,d)\in s(\mx)$.

%

Let $u=\min\{ \mx_{cb}, \mx_{ad}, v \}$. Let 
$\mw=\swap(\mx,a,c,b,d,u)$
, and we have
%

\begin{align}
  &\bullet ~\mw\text{~is $u$-swap distance from~}\mx,  \label{eq:trans:wxfirst}\\
  &\bullet ~\summw{i_1}{i_2}{j_1}{j_2} = \summx{i_1}{i_2}{j_1}{j_2}+u ~~ \bullet ~\summw{i_3}{i_4}{j_3}{j_4} = \summx{i_3}{i_4}{j_3}{j_4}+u,\\
  &\bullet ~\summw{i_1}{i_2}{j_3}{j_4} = \summx{i_1}{i_2}{j_3}{j_4}-u ~~ \bullet ~\summw{i_3}{i_4}{j_1}{j_2} = \summx{i_3}{i_4}{j_1}{j_2}-u.\\
  &\bullet ~\forall (i,j)\notin \twodinterval{i_1}{i_2}{j_1}{j_2} \cup \twodinterval{i_1}{i_2}{j_3}{j_4} \cup \twodinterval{i_3}{i_4}{j_1}{j_2} \cup \twodinterval{i_3}{i_4}{j_3}{j_4}, \mw_{ij}=\mx_{ij}. \label{eq:trans:wxlast}
\end{align}

If $u=v$, $\my=\mw$ is a feasible output. Otherwise, either $u=\mx_{cb}$ or $u=\mx_{ad}$, so either $\mw_{cb}$ or $\mw_{ad}$ becomes zero. Note that $(a,b)$ and $(c,d)$ are not in $\twodinterval{i_1}{i_2}{j_3}{j_4} \cup \twodinterval{i_3}{i_4}{j_1}{j_2}$, so we have $|r(\mw)|+|s(\mw)|\leq |r(\mx)|+|s(\mx)|-1$. Let $\my=\trans{\mw}{v-u}{\transpara{i_1}{i_2}{i_3}{i_4}{j_1}{j_2}{j_3}{j_4}}$, and it is a feasible output because of the following properties.
\begin{itemize}
  \item $\my\text{~is $(u+(v-u)=v)$-swap distance from~}\mx$,
  \item $\summy{i_1}{i_2}{j_1}{j_2} 
  = \summw{i_1}{i_2}{j_1}{j_2}+(v-u)
  = \summx{i_1}{i_2}{j_1}{j_2}+v$,
  \item $\summy{i_3}{i_4}{j_3}{j_4} 
  = \summw{i_3}{i_4}{j_3}{j_4}+(v-u)
  = \summx{i_3}{i_4}{j_3}{j_4}+v$,
  \item $\summy{i_1}{i_2}{j_3}{j_4} 
  = \summw{i_1}{i_2}{j_3}{j_4}-(v-u)
  = \summx{i_1}{i_2}{j_3}{j_4}-v$,
  \item $\summy{i_3}{i_4}{j_1}{j_2} 
  = \summw{i_3}{i_4}{j_1}{j_2}-(v-u)
  = \summx{i_3}{i_4}{j_1}{j_2}-v$,
  \item $\forall (i,j)\notin \twodinterval{i_1}{i_2}{j_1}{j_2} \cup \twodinterval{i_1}{i_2}{j_3}{j_4} \cup \twodinterval{i_3}{i_4}{j_1}{j_2} \cup \twodinterval{i_3}{i_4}{j_3}{j_4}, \my_{ij}=\mw_{ij}=\mx_{ij}$,
\end{itemize}
where the relation between $\my$ and $\mw$ follows from inductive assumption, and the relation between $\mw$ and $\mx$ follows from \Cref{eq:trans:wxfirst} to (\ref{eq:trans:wxlast}).

To implement $\transname$, it is sufficient to enumerate all the elements in $r(\mx)$ and $s(\mx)$, so the running time can be $O(\ell k)$.

\end{proof}

\subsection{Proof for \Cref{lem:proundx}} \label{sec:proof:proundx}

 Lemma \ref{lem:proundx} is a special case of Lemma \ref{lem:partial-round}, with an additional input constraint $\sum_{j' \in [0,k]}\mx_{ij'} \in [0,1)$. In the output $\my$, the $i$-th row only has one positive element $\my_{ij}$. In algorithm \ref{alg:proundx},  \Cref{alg:prs:1} makes $\my_{ij'}$ where $j'<j$ be zero,  \Cref{alg:prs:4} makes $\my_{ij'}$ where $j'>j$ be zero, and  \Cref{alg:prs:2} and  \Cref{alg:prs:3} make $\sum_{i'=1}^{i-1} \my_{i'j}$ to be an integer.

\begin{proof}[Proof for \Cref{lem:proundx}]

There are mainly two parts in this proof. The former part proves the properties of output $\my$ based on Lemma \ref{lem:trans}, and the latter part proves the parameters are valid when calling $\transname$ in Algorithm \ref{alg:proundx}. 

\paragraph{First part of the proof}
In $\proundx(\mx,i)$, 
%
$\summy{i}{i}{0}{j-1}$ only changes at  \Cref{alg:prs:1}, when it decreases by $\summy{i}{i}{0}{j-1}$, so we have $\summy{i}{i}{0}{j-1}=0$ at  \Cref{alg:prs:5}. $\summy{i}{i}{j+1}{k}$ does not change after  \Cref{alg:prs:4}, where it decreases by $\summy{i}{i}{j+1}{k}$, so we have $\summy{i}{i}{j+1}{k}=0$ at  \Cref{alg:prs:5}. Note that $\my_{ij} = \sum_{j' \in [0,k]}\mx_{ij'}$, so the output $\my$ satisfies
\begin{align*}
    \my_{ij} = \sum_{j' \in [0,k]}\mx_{ij'}
\end{align*}

$\summy{1}{i-1}{j}{j}$ does not change after  \Cref{alg:prs:2} or  \Cref{alg:prs:3}. If the algorithm goes to  \Cref{alg:prs:2}, $\summy{1}{i-1}{j}{j}$ decreases by $v+1$ or $\summy{1}{i-1}{j}{j}-t$, which means $\summy{1}{i-1}{j}{j}=\max\left\{\left(\left\lfloor\summy{1}{i-1}{j}{j}\right\rfloor-1\right), t\right\}$ at  \Cref{alg:prs:5}. If the algorithm goes to  \Cref{alg:prs:3}, $\summy{1}{i-1}{j}{j}$ increases by $v$, which means $\summy{1}{i-1}{j}{j}=\left\lceil\summy{1}{i-1}{j}{j}\right\rceil$ at  \Cref{alg:prs:5}. In both cases, $\summy{1}{i-1}{j}{j}$ is a non-negative integer at  \Cref{alg:prs:5}, so the output $\my$ satisfies 
\begin{align*}
    \sum_{i'=1}^{i-1} \my_{i'j} \in \Z_{\geq 0}
\end{align*}

We use three $\transname$ operations on $\my$ which initially equals to $\mx$, and we have $v\leq 1$ in each operation $\trans{\mx}{v}{\transpara{i_1}{i_2}{i_3}{i_4}{j_1}{j_2}{j_3}{j_4}}$, so $\my$ is $3$-swap distance from $\mx$.

We use $O(1)$ $\transname$ operations on $\my$, so the algorithm takes $O(\ell k)$ time.

\paragraph{Second part of the proof}
We also need to show $0\leq v\leq \min\{  \summx{i_3}{i_4}{j_1}{j_2}, \summx{i_1}{i_2}{j_3}{j_4}\}$ in each operation $\trans{\mx}{v}{\transpara{i_1}{i_2}{i_3}{i_4}{j_1}{j_2}{j_3}{j_4}}$ to ensure the operations are valid.

Based on the definition of $j$ and $t$, we have
\begin{align*}
    & \summx{1}{i-1}{0}{k}=\summx{1}{\ell}{0}{j-1}+t,  \\
    & 0\leq t < \summy{1}{\ell}{j}{j}.
\end{align*}

Note that $t$ and $\summy{1}{\ell}{j}{j}$ are both integers, so we have
\begin{align*}
    t \leq  \summy{1}{\ell}{j}{j}-1.
\end{align*}

Before  \Cref{alg:prs:1}, $\my=\mx$, and we have $\summy{1}{i-1}{0}{k}\geq \summy{1}{\ell}{1}{j-1}$, which implies
\begin{align*}
    \summy{1}{i-1}{j}{k}\geq \summy{i}{\ell}{0}{j-1}\geq \summy{i}{i}{0}{j-1},
\end{align*}
which implies  \Cref{alg:prs:1} is valid.

Before  \Cref{alg:prs:2}, we have $\summy{1}{i-1}{0}{k}= \summy{1}{\ell}{0}{j-1}+t$. Also, we have $\summy{i}{i}{0}{j-1}=0$ because of  \Cref{alg:prs:1}. So, we have 
\begin{align*}
    \summy{i+1}{\ell}{0}{j-1}\geq \summy{1}{i-1}{j}{j}- t,
\end{align*}
which implies  \Cref{alg:prs:2} is valid.

Before  \Cref{alg:prs:3}, we have $\summy{1}{\ell}{j}{j}\geq t$. Note that $t$ is an integer, so when $\summy{1}{i-1}{j}{j}< t$, we have
\begin{align*}
    \summy{i}{\ell}{j}{j}
    =\summy{1}{\ell}{j}{j}-\summy{1}{i-1}{j}{j}
    \geq t-\summy{1}{i-1}{j}{j}
    \geq v.
\end{align*}
We also have,
\begin{align*}
    \summy{1}{i-1}{j+1}{k}
    &=\summy{1}{i-1}{0}{k}-\summy{1}{i-1}{0}{j-1}-\summy{1}{i-1}{j}{j} \\
    &\geq \left(\summy{1}{i-1}{0}{k}-\summy{1}{\ell}{0}{j-1}\right)-\summy{1}{i-1}{j}{j} \\
    &= t-\summy{1}{i-1}{j}{j}
    \geq v.
\end{align*}
These two formulas imply  \Cref{alg:prs:3} is valid.

After  \Cref{alg:prs:2}, we have either $\summy{1}{i-1}{j}{j}=t$ (when $v+1<\summy{1}{i-1}{j}{j}-t$) or $\summy{i+1}{\ell}{j}{j}\geq 1$ (the other case). After  \Cref{alg:prs:3}, we have $\summy{1}{i-1}{j}{j}\leq t$. Since $t$ is an integer, when $\summy{1}{i-1}{j}{j}\leq t$, we have
\begin{align*}
    \summy{i}{\ell}{j}{j}
    =\summy{1}{\ell}{j}{j}-\summy{1}{i-1}{j}{j}
    \geq 1.
\end{align*}
As a result, we always have $\summy{i}{\ell}{j}{j}\geq 1$ before  \Cref{alg:prs:4}. Note that $\summy{i}{i}{0}{k}< 1$ and $\summy{i}{i}{0}{j-1}=0$, so we have
\begin{align*}
    \summy{i+1}{\ell}{j}{j}
    =\summy{i}{\ell}{j}{j} - \my_{ij}
    >\summy{i}{i}{0}{k} - \my_{ij}
    = \summy{i}{i}{j+1}{k},
\end{align*}
which implies  \Cref{alg:prs:4} is valid.
\end{proof}

\subsection{Proof for \Cref{lem:partial-round}}  \label{sec:proof:partial-round}

\begin{proof}[Proof of \Cref{lem:partial-round}]

We prove the properties mentioned in Lemma \ref{lem:partial-round} one by one.
Firstly we have
\begin{itemize}
  \item  $\my_{ij} \geq \mw_{(i+1),j} = \sum_{j' \in [0,k]}\mtz_{(i+1),j'} = o$,
  \item  $\sum_{i'=1}^{i-1} \my_{i'j}+\my_{ij}-o
   = \sum_{i'=1}^{i+1} \mw_{i'j}-o 
   = \sum_{i'=1}^{i} \mw_{i'j}
  \in \Z_{\geq 0}$,  \\
\end{itemize}
where the relation between $\mw$ and $\mtz$ is from Lemma \ref{lem:proundx}, and the relation between $\my$ and $\mw$ is from the definition of $\mcombinename$.

Next we show $\my$ is $(3v)$-swap distance from $\mx$. By \cref{lem:proundx}, there is a set of parameters $\{(a_s,b_s,c_s,d_s,\epsilon_i)\}_{s\in [t]}$, where $\sum_{s\in [t]}\epsilon_s=3v$, such that $A_s=\swap(A_{s-1},a_s,b_s,c_s,d_s,\epsilon_s)$ for $s\in [t]$, and $A_0=\mtz$, $A_t=\mw$.

\newcommand{\I}[1]{\mathbb{I}(#1)}

Let $a'_s=a_s-\I{a_s>i}$ and $c'_s=c_s-\I{c_s>i}$ for $i\in [t]$, where $\I{\cdot}$ is the indicator function. Let $A'_0=\mx$ and $A'_s=\swap(A'_{s-1},a'_s,b_s,c'_s,d_s,\epsilon_s)$ for $s\in [t]$. Then we have $A'_{t}=\my$, which implies $\my$ is $(3v)$-swap distance from $\mx$.

%

We proved all the conditions of the lemma and we concluded the proof.

\end{proof}
\section{Description and guarantees of $\roundi$}\label{subsec:roundi}
Here we provide a description of the subroutine $\roundi$.
\begin{algorithm}[H]
	\caption{$\roundi(\my,j,i)$}
	\begin{algorithmic}[1]
	    \State Let matrix $\mx \in \R_{\geq 0}^{\ell \times [0,k]}$.
	    \State Assign $\mx_{i'j'}=\my_{i'j'}$ for all $i' \in [\ell], ~j'\in [0,k] \text{ and }j' \neq j$; and $\mx_{i'j}=\my_{i'j}$ for all $i' \in [1,i-1]$. \label{algs:assignx=y-o:line:2} \;
		\State Assign $\mx_{ij}=\my_{ij}-o$ and $\res=1-o$.\label{algs:assignx=y-o:line:3}\;\label{algs:assignx=y-o}
		\For{$i'=i+1\dots \ell$}
		\State Let $w_{i'}=\min(\my_{i'j},\res)$.\;
		\State $\mx_{i'j}=\my_{i'j}-w_{i'}$ and $\res=\res-w_{i'}$.\label{algs:assignx=y-o:line:6}\;\label{algs:assignx=y-w} 
		\If{$\res=0$}
		\State exit the for loop.
		\EndIf
		\EndFor
		\State \textbf{Return} $\mx$.
	\end{algorithmic}
	\label{alg:final_temp}
\end{algorithm}

The guarantees of the above algorithm are summarized in the lemma below.
\lemroundi*
\begin{proof}
We first show that $\mx \in \R^{\ell \times [0,k]}$ is a non-negative matrix. Note that $\mx_{i'j'}=\my_{i'j'}-o$ if $i'=i$, $j'=j$ or else $\mx_{i'j'}$ takes one of the values in the set $\{\my_{i'j'},\my_{i'j'}-\min(\my_{i'j'},\res)\}$. As $\my_{ij} \geq o$, we get that $\mx_{i'j'} \geq 0$ for all $i' \in [\ell]$ and $j' \in [0,k]$. Furthermore, it is immediate that, $\mx \leq \my$.

Consider any $i' \in [1,i-1]$ and it is immediate that (\Cref{algs:assignx=y-o:line:2}) $[\mx \onevec]_{i'}=[\my \onevec]_{i'} \in \Z_{\geq 0}$. Furthermore, $[\mx \onevec]_{i}=[\my\onevec]_{i}-o=\lfloor [\my\onevec]_{i} \rfloor \in \Z_{\geq 0}$. 

Consider any $j'\in [0,k]$ such that $j' \neq j$ and note that $\mx_{i'j'}=\my_{i'j'}$ for all $i' \in [\ell]$. Therefore, we get that 
$[\mx^{\top} \onevec]_{j'}=[\my^{\top} \onevec]_{j'} \in \Z_{\geq 0}$. Now consider the index $j$,
\begin{align*}
    \sum_{i'\in [\ell]}\mx_{i'j}+1&=\sum_{i' \in  [1,i-1]}\mx_{i'j}+\sum_{i' \in  [i,\ell]}\mx_{i'j} + 1=\sum_{i' \in  [1,i-1]}\my_{i'j} + \sum_{i' \in  [i,\ell]}(\my_{i'j}-w_{i'}) + 1,\\
    &=\sum_{i' \in  [\ell]}\my_{i'j} -o- \sum_{i' \in  [i+1,\ell]}w_{i'} + 1
\end{align*}
The second and third equalities follow from Lines \ref{algs:assignx=y-w} and \ref{algs:assignx=y-o} of the algorithm respectively. 
In the following we show that, $\sum_{i' \in  [i+1,\ell]}w_{i'}=1-o$. To show this equality all we need is to show that $\res =0$ at the end of the loop, which holds when $\sum_{i' \in  [i+1,\ell]} \my_{i'j}\geq 1-o$. As $\sum_{i' \in [\ell]}\my_{i'j}\in \Z_{\geq 0}$, $o \in (0,1)_{\R}$ and $z \defeq \sum_{i'=1}^{i} \my_{i'j}-o \in \Z_{\geq 0}$, we get that,
\begin{align*}
    \sum_{i' \in [i+1,\ell]}\my_{i'j}&=\sum_{i' \in [\ell]}\my_{i'j}-\sum_{i' \in [1,i]}\my_{i'j}=\sum_{i' \in [\ell]}\my_{i'j}-z-o \text{ for some integer } z\in \Z_{\geq 0},
\end{align*}
In the following we show that $\sum_{i' \in [\ell]}\my_{i'j}-z \geq 1$. As $\sum_{i' \in [\ell]}\my_{i'j}-z \in \Z_{\geq 0}$ all we need is to show that $\sum_{i' \in [\ell]}\my_{i'j}-z>0$. Suppose $\sum_{i' \in [\ell]}\my_{i'j}-z=0$, then $\sum_{i' \in [i+1,\ell]}\my_{i'j}=\sum_{i' \in [\ell]}\my_{i'j}-z-o=-o <0$, which is a contradiction as $\my_{i'j'} \geq 0$ for all $i' \in [\ell]$, $j\in [0,k]$.
 Therefore we get that for all $j' \in [0,k]$,
\begin{equation}\label{eq11}
    \sum_{i' \in [\ell]}\mx_{i'j'}+1=\sum_{i' \in [\ell]}\my_{i'j'}~.
\end{equation}
As $\mx \in \R_{\geq 0}^{\ell \times [0,k]}$ and $\sum_{i' \in [\ell]}\my_{i'j} \in \Z_{\geq 0}$ (requirements of the lemma) we get that, $\sum_{i'\in [\ell]}\mx_{i'j} \in \Z_{\geq 0}$. Therefore,
$$\mx ^{\top} \onevec \in \Z^{[0,k]}_{\geq 0}~.$$
In the remainder of the proof we show that $\|\mx-\my\|_1 \leq 1$. Recall, earlier we showed that: $\mx \leq \my$, $[\mx^{\top}\onevec]_{j'}=[\my^{\top}\onevec]_{j'}$ for all $j'\neq j$ and $[\mx^{\top}\onevec]_{j}+1=[\my^{\top}\onevec]_{j}$. Combining these inequalities together, we immediately get that, $\| \mx-\my\|_1=1$. 

We proved all the conditions of the lemma and we conclude the proof.

\end{proof}

\section{Guarantees of $\swapround$}
To prove \Cref{lem:matrixround}, we show a stronger version of the lemma.

\begin{lem}\label{lem:matrixround-strong}
For any matrix $\ma \in \R^{s \times t}$ $(s \leq t)$ that satisfies $\ma^{\top} \onevec \in \Z^{t}_{\geq 0}$. In algorithm $\swapround$, for all $r\in \{0\}\cup [\ell]$,
\begin{itemize}
    \item $(\mdd^{(r)}+\ma^{(r)})$ is $(3r)$-swap distance from $\ma$.
    \item $\mdd^{(r)}_{ij}\geq 0, \ma^{(r)}_{ij}\geq 0$ for all $i \in [s]$ and $j \in [t]$, $[\ma^{(r)} \onevec]_i \in \Z_{\geq 0}$ for all $i\in [r]$, $(\ma^{(r)})^{\top} \onevec \in \Z_{\geq 0}^{t}$ and $\|\mdd^{(r)}\|_1 \leq r$.
\end{itemize} 
\end{lem}

\begin{proof}
We prove this lemma through induction. 
\paragraph{Base case i=0}
In this case $\ma^{(0)}=\ma$, $\mdd^{(0)}=0$, we have that all the conditions of the lemma are immediately satisfied.

\paragraph{Inductive step $r$}
We assume that the lemma conditions hold for $r-1$ and prove for the case $r$. We use $\my^{(r)}$ to denote the value of $\my$ in algorithm $\swapround$ when the enumerator is $r$. We verify the properties claimed by the lemma one by one, and we use the properties related to $\pround$ and $\roundi$ given by \Cref{lem:partial-round} and \Cref{lem:roundi} respectively.

Since $(\my^{(r)},j) = \pround(\ma^{(r-1)},r)$, $\my^{(r)}$ is $3$-swap distance from $\ma^{(r-1)}$, which implies $(\my^{(r)}+\mdd^{(r-1)})$ is $3$-swap distance from $(\ma^{(r-1)}+\mdd^{(r-1)})$. Note that $(\ma^{(r)}+\mdd^{(r)})=(\my^{(r)}+\mdd^{(r-1)})$, and $(\ma^{(r-1)}+\mdd^{(r-1)})$ is $(3(r-1))$-swap distance from $\ma$, so we have $(\mdd^{(r)}+\ma^{(r)})$ is $(3r)$-swap distance from $\ma$.

Since $\ma^{(r)} = \roundi (\my^{(r)},j,r)$, we have $\ma^{(r)}_{ij}\geq 0$ and  $(\my^{(r)}-\ma^{(r)})_{ij}\geq 0$ for all $i \in [s]$ and $j \in [t]$. Furthermore, since $\mdd^{(r-1)}_{ij}\geq 0$ for all $i \in [s]$ and $j \in [t]$, and $\mdd^{(r)}=(\my^{(r)}-\ma^{(r)})+\mdd^{(r-1)}$, we have $\mdd^{(r)}_{ij}\geq 0$ for all $i \in [s]$ and $j \in [t]$.

Since $(\my^{(r)},j) = \pround(\ma^{(r-1)},r)$, we have $[\my^{(r)} \onevec]_i=[\ma^{(r-1)} \onevec]_i \in \Z_{\geq 0}$ for $i\in [r-1]$.
Since $\ma^{(r)} = \roundi (\my^{(r)},j,r)$, we have $[\ma^{(r)} \onevec]_i = [\my^{(r)} \onevec]_i \in \Z_{\geq 0}$ for $i\in [r-1]$, and $[\ma^{(r)} \onevec]_i \in \Z_{\geq 0}$ for $i= r$. As a result, $[\ma^{(r)} \onevec]_i  \in \Z_{\geq 0}$ for $i\in [r]$.

Since $\ma^{(r)} = \roundi (\my^{(r)},j,r)$, we have $(\ma^{(r)})^{\top} \onevec \in \Z_{\geq 0}^{t}$.

Since $\ma^{(r)} = \roundi (\my^{(r)},j,r)$, we have $\|\my^{(r)}-\ma^{(r)}\|_{1}\leq 1$. Then we have $\|\mdd^{(r)}\|_{1}\leq \|\my^{(r)}-\ma^{(r)}\|_{1}+\|\mdd^{(r-1)}\|_{1}\leq r$ because $\mdd^{(r)}=(\my^{(r)}-\ma^{(r)})+\mdd^{(r-1)}$.

\end{proof}

\begin{cor}
For any matrix $\ma \in \R^{s \times t}$ $(s \leq t)$ that satisfies $\ma^{\top} \onevec \in \Z^{t}_{\geq 0}$. The algorithm $\swapround$ returns matrices $\ma'$ and $\mb$ such that,
\begin{itemize}
    \item $\ma'$ is $O(s)$-swap distance from $\ma$.
    \item $0 \leq \mb_{ij} \leq \ma'_{ij}$ for all $i \in [s]$ and $j \in [t]$, $\mb \onevec \in \Z_{\geq 0}^{s}$, $\mb^{\top} \onevec \in \Z_{\geq 0}^{t}$ and $\|\ma' -\mb\|_1 \leq O(s)$.
\end{itemize} 
\end{cor}

$\ma'$ is $O(s)$-swap distance from $\ma$ , and it implies $\ma' \onevec= \ma \onevec$ and $\ma'^{\top} \onevec= \ma^{\top} \onevec$.

We run $O(s)$ times $\pround$, $\roundi$, and matrix addition and subtraction, so the running time is $O(s^2 t)$.

We proved all the conditions of \Cref{lem:matrixround} and we conclude the proof.

%
\newcommand{\nri}{\vr'''_{i}}
\newcommand{\nr}{\vr'''}
\section{Proof for \Cref{lem:disc}}\label{app:disc}
\begin{proof}
Let $\{ \ri'\}_{i \in [\ell']}$ be the set of distinct probability values of distribution $\bp$. Let $\mx \in \R_{\geq 0}^{\ell'\times [0,k]}$ be the maximizer of the following optimization problem,
$$\max_{\my \in \textbf{Z}^{\phi}_{\bp}}\bg(\my) ~.$$
By \Cref{pmlprob:approx}, the maximizer $\mx$ satisfies,
$$\cphi \cdot \bg(\mx) \geq \exp(-O(\min(k,\ell')) \log n) \prob{\bp,\phi}~,$$
In the following we consider multiple distributions with different probabilities to denote the probability values explicitly, we extend the notation of $\bg(\mx)$ to $\bg(\mx,\vr)$. 
Given the solution $\mx$, we now maximize over the probability values, that is, solve,
\begin{align*}
    \max_{\vr'' \in [0,1]^{\ell'}} & \bg(\mx,\vr'')\defeq\sum_{i \in [\ell'],j\in [0,k]}(\mj \log \ri'' \mx_{ij} - \mx_{ij} \log \mx_{ij})+\sum_{i \in [\ell']}[\mx \vones]_i \log [\mx \vones]_i~,\\
    & \text{ such that }\sum_{i \in [\ell']}[\mx \vones]_i \ri''=1~.
\end{align*}
The optimum solution $\vr^*$ satisfies $\sum_{i \in [\ell']}[\mx \vones]_i \vr^*_i=1$ and,
$$\ri^{*}=\frac{\sum_{j \in [0,k]}\mj \mx_{ij}}{[\mx \onevec]_i\sum_{i' \in [\ell']}\sum_{j \in [0,k]}\mj \mx_{i'j}}=\frac{\sum_{j \in [0,k]}\mj \mx_{ij}}{ n[\mx \onevec]_i}~.$$
Furthermore,
\begin{equation}\label{eq:ab1}
    \cphi \bg(\mx,\vr^*) \geq \cphi \bg(\mx,\vr') \geq \exp(-O(\min(k,\ell')) \log n) \prob{\bp,\phi}~.
\end{equation}
Substituting values of $\vr^*$ in $\bg(\mx,\vr^*)$ we get,
\begin{equation}\label{eq:ab2}
\bg(\mx,\vr^*)=n \sum_{i \in [\ell']}[\mx \onevec]_i\ri^{*} \log \ri^{*} - \sum_{i \in [\ell'],j\in [0,k]}\mx_{ij} \log \mx_{ij}+\sum_{i \in [\ell']}[\mx \vones]_i \log [\mx \vones]_i    
\end{equation}
We now construct a pseudo distribution $\bq$ by rounding down the probability values into set $\bR$ as follows: for each $i \in [\ell]$, the number of elements with probability value $\ri$ is equal to $\sum_{\{i'\in [\ell']~|~\lfloor\vr_{i'}^{*} \rfloor_{\bR}=\ri\}}[\mx \onevec]_{i'}$, where $\lfloor y \rfloor_{\bR} \defeq \max_{\{x \in \bR~|~x \leq y \}}x$. Define a new solution $\mx' \in \R^{[\ell] \times [0,k]}$ as follows: $\mx'_{ij}\defeq \sum_{\{i'\in [\ell']~|~\lfloor\vr_{i'}^{*} \rfloor_{\bR}=\ri\}}\mx_{i'j}$. 
Let $\beta=\sum_{i' \in [\ell']}\lfloor\vr_{i'}^{*} \rfloor_{\bR} [\mx \onevec]_{i'}=\sum_{i \in [\ell]}\ri [\mx' \onevec]_{i}$ and note that $\|\bq\|_1=\beta$. Now note that the following inequality immediately holds,
\begin{equation}\label{eq:ab3}
\prob{\frac{\bq}{\|\bq\|_1},\phi}=\prob{\bq,\phi} \beta^{-n} \text{ and }~\prob{\bq,\phi} \geq \cphi \bg(\mx',\vr).    
\end{equation}
Now consider,
\begin{align}\label{eq:f1}
\log \bg(\mx',\vr)
&=\sum_{i \in [\ell],j\in [0,k]}(\mj \log \ri  \mx'_{ij} - \mx'_{ij} \log \mx'_{ij})+\sum_{i \in [\ell]}[\mx' \vones]_i \log [\mx' \vones]_i, \nonumber \\
& \geq \sum_{i \in [\ell'],j\in [0,k]}\mj \log \lfloor\vr_{i'}^{*} \rfloor_{\bR}  \mx_{ij} - \sum_{i \in [\ell'],j\in [0,k]} \mx_{ij} \log \mx_{ij}+\sum_{i \in [\ell']}[\mx \vones]_i \log [\mx \vones]_i~.
\end{align}
Consider the first term in the above expression,
\begin{align}\label{eq:o1}
    \sum_{i \in [\ell'],j\in [0,k]}\mj \log \lfloor\vr_{i'}^{*} \rfloor_{\bR}  \mx_{ij} & = \sum_{i \in [\ell'],j\in [0,k]}\mj \log  \frac{\lfloor\ri^{*}\rfloor_{\bR}}{\beta}  \mx_{ij}+\log \beta \sum_{i \in [\ell'],j\in [0,k]}\mj \mx_{ij}, \nonumber \\ 
    &=\sum_{i \in [\ell'],j\in [0,k]}\mj \log  \frac{\lfloor\ri^{*}\rfloor_{\bR}}{\beta}  \mx_{ij}+ n \log \beta 
\end{align}
To simplify the above expression we define $\alpha_i$ as, $\ri^{*}(1+\alpha_{i}) =\frac{\lfloor\ri^{*}\rfloor_{\bR}}{\beta}$, then we get that, 
$$\sum_{i\in [\ell']}\ri^{*}(1+\alpha_{i}) [\mx \onevec]_{i}=\sum_{i\in [\ell']}\frac{\lfloor\ri^{*}\rfloor_{\bR}}{\beta} [\mx \onevec]_{i}=1.$$
However, we also have that, $\sum_{i\in [\ell']}\ri^{*} [\mx \onevec]_{i}=1$. Combining both we get that, 
$\sum_{i\in [\ell']}\alpha_i \ri^* [\mx\onevec]_{i}=0$. Substituting this equality in Equation \ref{eq:o1}, we get,
\begin{align}\label{eq:o2}
    \sum_{i \in [\ell'],j\in [0,k]}& \mj \log \lfloor\vr_{i'}^{*} \rfloor_{\bR}  \mx_{ij}
    =\sum_{i \in [\ell'],j\in [0,k]}\mj \log  \frac{\lfloor\ri^{*}\rfloor_{\bR}}{\beta}  \mx_{ij}+ n \log \beta \nonumber\\
    & =\sum_{i \in [\ell'],j\in [0,k]}\mj \log  \ri^{*}(1+\alpha_i)  \mx_{ij}+ n \log \beta \nonumber\\
    & =\sum_{i \in [\ell'],j\in [0,k]}\mj \log  \ri^{*}  \mx_{ij}+\sum_{i \in [\ell'],j\in [0,k]}\mj \log (1+\alpha_i)  \mx_{ij} +n \log \beta \nonumber\\
    & \geq \sum_{i \in [\ell'],j\in [0,k]}\mj \log  \ri^{*}  \mx_{ij}+\sum_{i \in [\ell'],j\in [0,k]}\mj (\alpha_i-\alpha_i^2)  \mx_{ij} +n \log \beta \nonumber\\
    & \geq \sum_{i \in [\ell'],j\in [0,k]}\mj \log  \ri^{*}  \mx_{ij}-\sum_{i \in [\ell'],j\in [0,k]}\mj \alpha_i^2 \mx_{ij} +n \log \beta~.
\end{align}
In the fourth inequality we used the inequality $\log (1+x) \geq x -x^2$, when $x \in (0,1)$. In the last inequality we used $\sum_{i\in [\ell']}\alpha_i \ri^* [\mx\onevec]_{i}=0$. Now note that each of the $\alpha_i$ satisfy, $|\alpha_{i}| \leq O(\alpha)$. Substituting it in the above expression we get,
\begin{align}\label{eq:o2}
    \sum_{i \in [\ell'],j\in [0,k]}\mj \log \lfloor\vr_{i'}^{*} \rfloor_{\bR}  \mx_{ij}
    & \geq \sum_{i \in [\ell'],j\in [0,k]}\mj \log  \ri^{*}  \mx_{ij}-\alpha^2 n +n \log \beta,
\end{align}
where in the above inequality we used $\sum_{i \in [\ell'],j\in [0,k]}\mj \mx_{ij} = \sum_{j\in [0,k]}\mj \phi_j =n$. Combining the above inequality with \Cref{eq:f1}, we get, 
\begin{align}\label{eq:ab5}
    \log &\bg(\mx',\vr)  \geq \sum_{i \in [\ell'],j\in [0,k]}\mj \log \lfloor\vr_{i'}^{*} \rfloor_{\bR}  \mx_{ij} - \sum_{i \in [\ell'],j\in [0,k]} \mx_{ij} \log \mx_{ij}+\sum_{i \in [\ell']}[\mx \vones]_i \log [\mx \vones]_i, \nonumber\\
    & \geq \sum_{i \in [\ell'],j\in [0,k]}\mj \log  \ri^{*}  \mx_{ij}- \sum_{i \in [\ell'],j\in [0,k]} \mx_{ij} \log \mx_{ij}+\sum_{i \in [\ell']}[\mx \vones]_i \log [\mx \vones]_i-\alpha^2 n +n \log \beta, \nonumber \\
    & = \log \bg(\mx,\vr^{*})-\alpha^2 n +n \log \beta
\end{align}
Combining everything we get the following,
\begin{align*}
\log \prob{\bp,\phi} - & \log \prob{\frac{\bq}{\|\bq\|_1},\phi} 
 = \log \prob{\bp,\phi} - \log \prob{\bq,\phi} + n \log \beta,\\
& \leq O(\min(k,\ell') \log n) + \log (\cphi \bg(\mx,\vr^*)) - \log (\cphi \bg(\mx',\vr)) + n \log \beta,\\
& \leq O(\min(k,\ell') \log n) + \alpha^2 n~.
\end{align*}
In the first, second and third inequalities we use \Cref{eq:ab1}, \Cref{eq:ab3} and \Cref{eq:ab5} respectively. 
We conclude the proof.
\end{proof}
\renewcommand{\cones}{\onevec}
\section{Proof for \Cref{lem:swap}}

We firstly show a special case of \Cref{lem:swap}.

\begin{lem}\label{lem:swap-special}

If $\ma'=\swap(\ma,i_1,i_2,j_1,j_2,\epsilon)$ where $\ma,\ma' \in \R^{\ell \times [0,k]}$ and $\ma \in \bZfrac$, then,
    $\ma' \in \bZfrac \text{ and } ~~\bg(\ma') \geq \exp(-4\epsilon \log n) ~\bg(\ma)$.

\end{lem}

\begin{proof}

Since $\ma'=\swap(\ma,i_1,i_2,j_1,j_2,\epsilon)$, we have $\ma'^\top \cones=\ma^\top \cones$ and $\ma' \cones=\ma \cones$. As $\ma \in \bZfrac$, we immediately get that $\ma' \in \bZfrac$ and we also have
\begin{align*}
    \frac{\bg(\ma')}{\bg(\ma)}
    &= \frac{\exp\Big(\sum_{i\in[\ell],j\in[0,k]}\left[\mc_{ij}\ma'_{ij}-\ma'_{ij}\log\ma'_{ij}\right]+\sum_{i\in[\ell]}[\ma'\vones]_{i}\log[\ma'\vones]_{i}\Big)}{\exp\Big(\sum_{i\in[\ell],j\in[0,k]}\left[\mc_{ij}\ma_{ij}-\ma_{ij}\log\ma_{ij}\right]+\sum_{i\in[\ell]}[\ma\vones]_{i}\log[\ma\vones]_{i}\Big)}  \\
    =& \frac{\exp\Big(\sum_{i\in[\ell],j\in[0,k]}\left[\mc_{ij}\ma'_{ij}-\ma'_{ij}\log\ma'_{ij}\right]\Big)}{\exp\Big(\sum_{i\in[\ell],j\in[0,k]}\left[\mc_{ij}\ma_{ij}-\ma_{ij}\log\ma_{ij}\right]\Big)}  \\
    \geq & \frac{ \exp\Big(   \left(\mc_{i_1j_1} - \mc_{i_1j_2} - \mc_{i_2j_1} + \mc_{i_2j_2} \right)\epsilon  
    \Big) 
    }
    { \exp\Big(  \ma'_{i_1j_1}\log\ma'_{i_1j_1} -  \ma_{i_1j_1}\log\ma'_{i_1j_1}
     + \ma'_{i_2j_2}\log\ma'_{i_2j_2}-  \ma_{i_2j_2}\log\ma'_{i_2j_2}
     \Big) }    \\
     \geq & \frac{ 1 }{\exp \Big(  2\epsilon\log(n+1) \Big)} \geq \exp \Big(  -4\epsilon\log n \Big),
\end{align*}
where the last inequality holds because because $\frac{d}{dx}\left(x\log(x)\right)=(\log (x)+1)$ and
\begin{align*}
    \mc_{i_1j_1} - \mc_{i_1j_2} - \mc_{i_2j_1} + \mc_{i_2j_2} 
    &= \mja\ria  - \mjb\ria - \mja\rib + \mjb\rib \\
    &= (\mjb  - \mja) ( \rib - \ria) \geq 0.
\end{align*}

\end{proof}

\begin{proof}[Proof for \Cref{lem:swap}]

Since $\ma'$ is $x$-swap distance from $\ma$, there exists a set of parameters denoted by 
$\{(i_1^{(s)},i_2^{(s)},j_1^{(s)},j_2^{(s)},\epsilon^{(s)})\}_{s\in [t]}$,
where $\sum_{s\in [t]}\epsilon^{(s)}\leq x$, s.t $\ma^{(s)}=\swap(\ma^{(s-1)},i_1^{(s)},i_2^{(s)},j_1^{(s)},j_2^{(s)},\epsilon^{(s)})$ for $s\in [t]$, where $\ma^{(0)}=\ma$, $\ma^{(t)}=\ma'$.

We apply \Cref{lem:swap-special} on $\ma^{(i)}$ and $\ma^{(i-1)}$ for each $i\in [t]$, and we have
\begin{align*}
    \ma^{(0)}  \in \bZfrac  \Rightarrow 
    \ma^{(1)}  \in \bZfrac  \Rightarrow 
    \ma^{(2)}  \in \bZfrac  \Rightarrow 
    \dots  \Rightarrow 
    \ma^{(t)}  \in \bZfrac ,
\end{align*}
\begin{align*}
    \bg(\ma^{(t)})
    \geq \exp\left(-4\epsilon^{(t)} \log n ~ \right)\bg(\ma^{(t-1)})
    \geq \dots
    \geq \exp\left(-4\left(\sum_{s\in [t]}\epsilon^{(s)}\right) \log n\right) ~\bg(\ma^{(0)}),
\end{align*}
which implies $\ma'  \in \bZfrac$ and $\bg(\ma')\geq \exp\left(-O\left(x \log n\right)\right) ~\bg(\ma)$. We conclude the proof.
\end{proof}
\end{document}